\documentclass[11pt]{scrartcl}

\usepackage[utf8]{inputenc}
\usepackage[linesnumbered,ruled,vlined]{algorithm2e}
\usepackage{nameref}
\usepackage[colorlinks=true, urlcolor=blue, citecolor=blue, linkcolor=blue]{hyperref}
\usepackage{amsfonts, amsmath, amssymb, amsthm}
\usepackage{thmtools}
\usepackage{fancyvrb}
\VerbatimFootnotes
\usepackage{enumerate}
\usepackage{tabularx}
\usepackage{adjustbox}
\usepackage{booktabs}
\usepackage{color, enumitem, fancyhdr, float, graphicx, subfig}
\usepackage[all]{xy}
\graphicspath{{./figures/}{../figures/}}
\usepackage{ifthen}
\usepackage{courier}
\usepackage{authblk} 
\usepackage{natbib}

\newtheorem{theorem}{Theorem}[section]

\newtheorem{assumption}{Assumption}[section]

\usepackage[a4paper,
            textheight=23cm,
            top=2cm, 
            bottom=2cm, 
            textwidth=15cm,
            hscale=0.8, 
            footskip=1.5cm,
            headsep=0.5cm]{geometry}
\usepackage{setspace}
\newboolean{showSolutions}
\setboolean{showSolutions}{true}
\newcommand{\solutioncmd}[2]{{{#1}}{{#2}}}
\newcommand{\solution}[2]{\ifthenelse {\boolean{showSolutions}} {\solutioncmd{#1}}{\solutioncmd{#2}}}

\newcommand{\E}{{\mathbb{E}}}

\newcommand{\I}{{\mathbb{I}}}

\newcommand{\sY}{{\mathcal{Y}}}

\newcommand{\x}{\mathbf{x}}

\newcommand{\X}{\mathbf{X}}

\renewcommand{\P}{{\mathbb{P}}}
\newcommand{\ourmethod}{\texttt{CP4SBI}}
\newcommand{\cdfourmethod}{\texttt{CDF CP4SBI}}
\newcommand{\locartourmethod}{\texttt{LoCart CP4SBI}}

\newcommand{\bcdfourmethod}{\texttt{\textbf{CDF CP4SBI}}}
\newcommand{\blocartourmethod}{\texttt{\textbf{LoCart CP4SBI}}}

\title{CP4SBI: Local Conformal Calibration of Credible Sets in Simulation-Based Inference}
\author[1, 2, 3]{Luben M. C. Cabezas}
\author[1]{Vagner S. Santos}
\author[1]{Thiago R. Ramos}
\author[3]{Pedro L. C. Rodrigues}
\author[1]{Rafael Izbicki}
\affil[1]{Department of Statistics, Federal University of São Carlos}
\affil[2]{Institute of Mathematics and Computer Science, University of São Paulo}
\affil[3]{Univ.~Grenoble Alpes, Inria, CNRS, Grenoble INP, LJK}

\date{}
\begin{document}
\maketitle
\vspace{-4em}
\begin{center}
\textbf{Abstract}
\end{center}

\vspace{-1.5em} 

\begin{quote}
Current experimental scientists have been increasingly relying on simulation-based inference (SBI) to invert complex non-linear models with intractable likelihoods.
However, posterior approximations obtained with SBI are often miscalibrated, causing credible regions to undercover true parameters. We develop \ourmethod{}, a model-agnostic conformal calibration framework that constructs credible sets with local Bayesian coverage. Our two proposed variants, namely local calibration via regression trees 
and CDF-based calibration,
enable finite-sample local coverage guarantees for any scoring function, including HPD, symmetric, and quantile-based regions. Experiments on widely used SBI benchmarks demonstrate that our approach improves the quality of uncertainty quantification for neural posterior estimators using both normalizing flows and score-diffusion modeling.
\end{quote}


\vspace{1em}

\section{Introduction}
\label{sec:introduction}
In recent years, the machine learning research community has shown growing interest in tackling challenging open problems across a range of experimental sciences, including biology \citep{chicco2017ten, min2017deep}, astrophysics \citep{freeman2017unified, rodriguez2022application}, neuroscience \citep{lueckmann2017flexible}, and fluid dynamics \citep{usman2021machine, vinuesa2022enhancing}. A fundamental challenge common to all these domains—and the focus of this work—is the task of inferring the parameter values of statistical models that best explain observed data \citep{tarantola2005inverse, casella2024statistical}. These statistical models are typically formalized as \emph{stochastic simulators}, 
which generate synthetic data $\x \in \mathcal{X}$ given input \emph{parameters} $\theta \in \Theta$. In other words, $\theta$ represents a set of parameters in these scientific models that govern the data $\mathbf{x}$ (e.g., physical parametrizations, infection rates, or statistical means). Although it is easy to generate data from such simulators,
these models generally lack a tractable likelihood function, 
making classical Bayesian inference techniques for computing posterior distributions, such as Markov Chain Monte Carlo (MCMC), infeasible.


Simulation-based inference (SBI) \citep{cranmer2020frontier} is a powerful framework for likelihood-free Bayesian inference, allowing posterior estimation from synthetic data generated by simulators. Recent advances in deep generative models enable SBI methods to approximate posterior distributions using expressive neural density estimators. Despite notable empirical successes, SBI is still maturing and faces key challenges that hinder wider adoption. These include sensitivity to model misspecification~\citep{Wehenkel2025}, difficulties with hierarchical or weakly identified models~\citep{rodrigues2021hnpe}, and limited theoretical understanding of posterior consistency with finite data~\citep{Linhart2023}. A meta-analysis by \citealt{hermans2022crisis} highlights that these limitations often result in overconfident and miscalibrated posterior estimates.
In particular, credible regions of the form
\begin{align*}
C(\mathbf{x}_{\text{obs}}) = \{\theta : \widehat{p}(\theta \mid \x_{\text{obs}}) \geq K \},
\end{align*}
constructed from approximate posteriors $\widehat{p}(\theta \mid \x_{\text{obs}})$, where $\x_{\text{obs}}$ is the observed data, may fail to achieve their intended coverage levels. This issue, further explored in Section~\ref{sec:bayesian_credible}, undermines the reliability of SBI in downstream decision-making tasks~\citep{murphy2022probabilistic}.


To mitigate the issues of miscalibrated posterior distributions, recent work draws on the success of
Conformal prediction (CP) for prediction methods~\citep{shafer2008tutorial, angelopoulos2023conformal} to construct credible regions with finite-sample coverage guarantees in an SBI setting. Concretely, these papers apply conformal prediction to a calibration dataset $\mathcal{D}_{\text{cal}} = \{(\theta_i, \X_i)\}_{i=1}^B$, where each $(\theta, \X)$ is simulated from the prior and the model.
However, vanilla CP yields only  marginally calibrated regions \citep{patel2023variational, baragatti2024approximate}, meaning that
\begin{align}
\label{eq:marginal_coverage}
    \P(\theta \in C(\X)) = 1 - \alpha,
\end{align}
where the probability is taken jointly over random draws of the parameter $\theta$, the data $\X$, and a calibration dataset used to construct $C$. This property is often referred to as marginal coverage.

 This is less informative than \emph{conditional coverage} for a given observation $\x_\text{obs}$, defined as
\begin{align}
\label{eq:cond_coverage}
    \P(\theta \in C(\X)\mid\X = \x_{\text{obs}}) = 1 - \alpha,
\end{align}
a property that is central to Bayesian inference. Moreover, current CP methods require access to a closed-form approximation to the posterior $p(\theta\mid\x)$, which is not necessarily possible for all classes of generative models currently used for SBI, such as diffusion models~\citep{linhart2024diffusion} and flow matching~\citep{wildberger2023flow}. 

We address this gap by incorporating \emph{local}  CP techniques \citep{izbicki2022cd, cabezas2025regression, dheur2025multi} into the SBI framework. Our method, \ourmethod{}, supports a broad range of posterior approximators, including both density- and sample-based approaches. As illustrated in Figure~\ref{fig:illustration_locart}, \ourmethod{} yields more reliable, observation-specific credible regions, improving calibration across different $\x_{\text{obs}}$ and refining standard globally calibrated methods. It only requires a calibration dataset $\mathcal{D}_{\text{cal}} = \{(\theta_i, \x_i)\}_{i=1}^B$, where each $(\theta, \X)$ is either simulated from the prior and the model, or obtained by withholding a small portion (e.g., 20\%) of the training data used for the posterior approximator.

\begin{figure}[h]
    \centering
    \includegraphics[width=1\columnwidth]{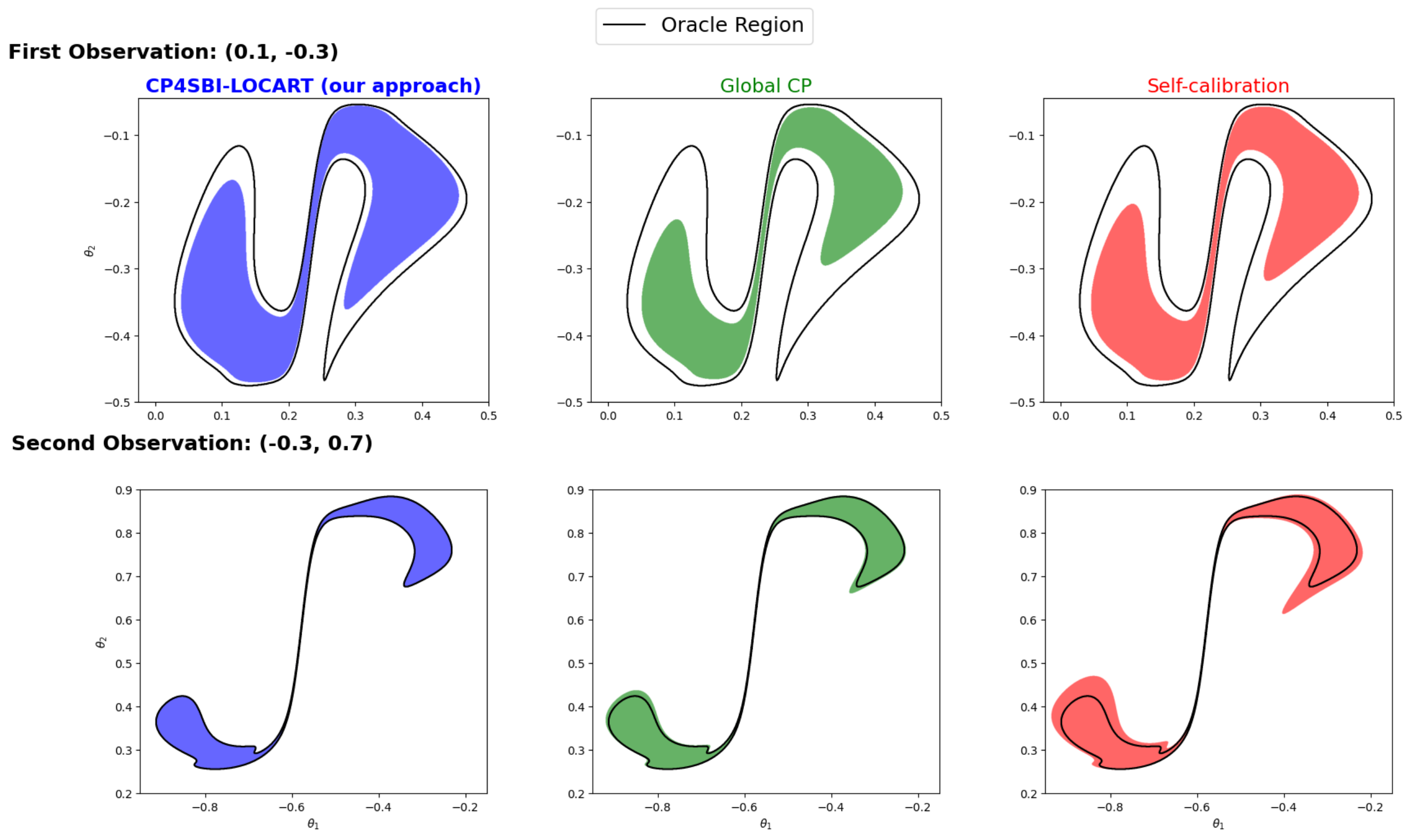}
    \caption{Comparison of highest posterior density credible regions on the \textit{Two Moons} benchmark for two distinct observations $\x_{\text{obs}}$. The Two Moons benchmark is a two-parameter task that exhibits bimodality and crescent shapes, for which we know the true posterior. Here, $\theta$ controls the location of the crescent-shaped data.  The proposed \locartourmethod{} method (blue) produces regions that more closely match the oracle regions (black) with correct coverage, compared to the global conformal (green) and self-calibrated (red) baselines. This highlights the benefits of local adaptivity in improving the quality of credible region estimates.}
    \label{fig:illustration_locart}
\end{figure}

\textbf{Paper Outline.} Section \ref{sec:related_works} reviews related work. Section \ref{sec:novelty} presents our  contributions. Background is shown in Section \ref{sec:background}, followed by methodology in Section \ref{sec:methodology}. Section \ref{sec:theory} presents theoretical results, and Section \ref{sec:experiments} shows experiments on SBI benchmarks. Section \ref{sec:conclusions} concludes the paper. Appendices \ref{appendix:algorithms} to \ref{appendix:exps} contain algorithms, proofs, and additional results and details.

\subsection{Relation to Other Work}
\label{sec:related_works}

\textbf{Approximations of posterior distributions.}  Various SBI strategies for posterior estimation have been developed over recent years. These include initial approaches like Approximate Bayesian Computation (ABC) \citep{marin2012approximate,sisson2018handbook}, which approximates the posterior for a specific observation $\x_{\text{obs}}$ via rejection sampling using a specified distance function and rejection sampling or other MCMC methods. Following this,  amortized methods started being used to estimate the posterior distribution and related quantities more efficiently \citep{cranmer2015approximating,gutmann2016bayesian,izbicki2019abc}. These methods include neural posterior estimation (NPE) \citep{papamakarios2016fast, lueckmann2017flexible, greenberg2019automatic, deistler2022truncated}, which directly approximates the posterior density $p(\theta \mid \x)$; neural likelihood estimation (NLE) \citep{papamakarios2019sequential, frazier2023bayesian}, which estimates the intractable likelihood function $p(\x \mid \theta)$; and  density ratio estimators \citep{izbicki2014high,hermans2020likelihood, durkan2020contrastive,dalmasso2020confidence,DalmassoLF2I}, which estimate likelihood ratios such as $p(\x \mid \theta)/p(\x)$. Additionally, a subsequent development in the field has been the adaptation of more implicit generative models, such as flow-matching and diffusion models, to SBI \citep{wildberger2023flow, geffner2023compositional, linhart2024diffusion}. These advanced techniques aim to faithfully generate samples directly from the posterior $p(\theta\mid\x)$ without requiring the estimation of a closed-form posterior density.
Our approach uses such methods, but provides means to better calibrate credible sets derived from them.
\vspace{2mm}


\textbf{Calibration of approximators and recalibration techniques.} Despite the breadth of SBI's recent research and its impact on various scientific fields, \citealt{hermans2022crisis} has revealed a prevalent challenge to all of its variants: they can suffer from statistical miscalibration and produce overconfident estimates in practical settings. Consequently, a significant body of recent work \citep{hermans2022crisis, dey2022conditionally, bon2022bayesian, delaunoy2022towards, chung2024sampling} has focused on improving posterior density calibration, with two main approaches:
(1)  Post-hoc procedures that recalibrate existing estimates, e.g., using local Probability Integral Transform (PIT) diagnostics \citep{zhao2021diagnostics, dey2022conditionally}, or adjusting samples and probability estimates with Highest Density Region (HDR) corrections \citep{chung2024sampling}; and 
    (2) Pre-fitting procedures that integrate calibration during the estimation process, e.g., ensembling posterior estimations \citep{hermans2022crisis} or regularizing estimator formulations \citep{delaunoy2022towards, delaunoy2023balancing}.     
The main limitation of methods from both categories is that they do not offer explicit finite-sample guarantees for the validity of credible regions.

\vspace{2mm}
\textbf{Conformal Prediction applied to approximate posteriors.}  Standard Conformal Prediction (CP) techniques, when applied to SBI, primarily provide post-hoc marginal coverage guarantees. Existing work in this direction includes \citealt{baragatti2024approximate}, which adapts CP for Approximate Bayesian Computation methods, and \citealt{patel2023variational}, which explores its use in variational inference —a setting, like SBI, where one also works with estimates of the posterior distribution.  Although these methods are effective in improving marginal coverage, they do not provide locally or approximately conditionally valid credible regions in general settings. Moreover, \citealt{patel2023variational} requires a closed-form approximation of the posterior, which is typically unavailable in the generative models used for SBI, where only sample-based estimates are provided.

\vspace{2mm}

\textbf{Frequentist confidence sets for SBI.} 
An emerging line of research aims to construct confidence sets with frequentist coverage in SBI settings, i.e., ensuring that $\P(\theta \in C(\X)\mid \theta) = 1 - \alpha$ for all $\theta \in \Theta$ as opposed to Equation~\eqref{eq:cond_coverage}. This includes methods based on Monte Carlo sampling \citep{cranmer2015approximating}, quantile regression \citep{dalmasso2020confidence, DalmassoLF2I, masserano2023simulator,carzon2025trustworthy,carzon2025focusing}, and conformal prediction \citep{cabezas2024distribution}. 
Our work does not target frequentist coverage but instead focuses on improving Bayesian posterior calibration.

\subsection{Novelty}
\label{sec:novelty}

We introduce \ourmethod{}, a framework that integrates conformal prediction techniques into the simulation-based inference pipeline to produce calibrated credible regions. Our main contributions are:

\begin{itemize}
    \item \textbf{Guaranteed Marginal Coverage:} Unlike standard credible sets from approximate posteriors which often suffer from miscalibration \citep{hermans2022crisis}, \ourmethod{} provides non-asymptotic guaranteed marginal coverage (Equation \ref{eq:marginal_coverage}) by reinterpreting Bayesian scores as nonconformity scores.   
    
    \item \textbf{Enhanced Local and Conditional Adaptivity:} We move beyond simple marginal guarantees, achieving stronger coverage criteria:
        \begin{itemize}
            \item \textbf{\bcdfourmethod{}:} Achieves asymptotic conditional coverage as the estimate $\widehat{p}(\theta\mid\x)$ gets closer to the true posterior distribution.  This implies that  $\P(\theta \in C(\X)\mid\X = \x_{\text{obs}})$ typically gets closer to $ 1 - \alpha $ as the number of simulated training samples increases. This is achieved by efficiently estimating the conditional CDF of scores, reusing the already-trained posterior approximation.
            \item \textbf{\blocartourmethod{}:} Provides local coverage by partitioning the data space with a regression tree, adapting credible region sizes based on inference difficulty. That is, \blocartourmethod\ achieves $ \P(\theta \in C(\X)\mid\X \in A) = 1 - \alpha $, where $A$ is a subset of $\mathcal{X}$ of datasets close to $ \x_{\text{obs}}$ according to a data-driven metric. Moreover,
     it  offers asymptotic conditional coverage in the sense that  $\P(\theta \in C(\X)\mid\X = \x_{\text{obs}})$  gets closer to $1-\alpha$ as the number of \emph{calibration} samples $B$ increases.
        \end{itemize}
    \item \textbf{Generality and Flexibility:} \ourmethod{} is a model-agnostic framework that can be applied on top of any posterior approximator, including density-based methods (e.g., NPE) and sample-based generative models (e.g., diffusion models, flow matching). Furthermore, it is compatible with various scoring functions, allowing for the construction of calibrated Highest Posterior Density regions, symmetric regions, or other custom credible sets. It also accommodates nuisance parameters and parameter transformations. This flexibility makes it a broadly applicable tool for enhancing the reliability of  SBI methods.
\end{itemize}

\section{Background}
\label{sec:background}
\subsection{Conformal Prediction}
\label{sec:conformal}

Conformal methods have recently emerged as a powerful framework for constructing prediction regions under minimal assumptions \citep{vovk2005algorithmic, shafer2008tutorial, angelopoulos2023conformal,Izbicki2025}. Given exchangeable data $\{(Y_i, \X_i)\}_{i=1}^{m+1}$, these methods construct a prediction set $C(\cdot)$ using the first $m$ pairs such that
$
\P\big(Y_{m+1} \in C(\X_{m+1})\big) \geq 1 - \alpha ,
$ where the probability is taken with respect to the joint distribution of all $m+1$ samples.

A standard split-conformal method begins by defining a nonconformity score $s: \mathcal{Y} \times \mathcal{X}\to \mathbb{R}$, which quantifies how well a candidate output $y \in \sY$ conforms to the input $\x \in \mathcal{X}$ given a fitted regression model. 
In regression problems, a common choice of score is the absolute residual
$
s(y; \x) = \big|y - \widehat{\E}[Y \mid \x]\big|,
$
where $\widehat{\E}[Y \mid \x]$ is an estimate of the regression function $\E[Y \mid \x]$ fitted on a subset of the labeled data reserved for training.
The prediction region then takes the form
$
C(\x) = \{y \in \mathcal{Y} : s(y; \x) \leq t_{1-\alpha}\},
$
where $t_{1-\alpha}$ is the $(1 + 1/n)(1 - \alpha)$-quantile of the conformity scores $s(Y_i; \X_i)$ evaluated on a calibration set of size $n$ disjoint from the training data used to construct the nonconformity score $s$.

Conformal methods guarantee marginal coverage \citep{papadopoulos2008normalized, Vovk2012, Lei2014, valle2023quantifying}, but generally do not ensure conditional coverage. Without strong assumptions on the data-generating distribution, achieving exact conditional coverage either requires trivial (often unbounded) prediction sets or results in coverage falling below the target level for some covariate values \citep{Lei2014, Barber2021}. 
To address this, several recent methods aim to achieve asymptotic conditional coverage as $m \to \infty$. One approach is to construct conformity scores whose conditional distribution given $\X$ is approximately independent of $\X$. Examples include conformalized quantile regression \citep{Romano2019}, distributional conformal prediction \citep{chernozhukov2019distributional}, Dist-split \citep{izbicki2019distribution}, HPD-split \citep{izbicki2022cd}, EPICSCORE \citep{cabezasepistemic}, and \citealt{plassier2025rectifying}.
In this work, we are particularly interested in the CDF-conformal score introduced by \citet[Eq.~14]{dheur2025multi}. Given a nonconformity score $s(y;\x)$ and an estimate $\widehat{F}(\cdot \mid \x)$ of its conditional cumulative distribution function, this approach defines the transformed score.
\[
s'(y; \x) = \widehat{F}(s(y; \x) \mid \x).
\]
This transformation improves conditional coverage by ensuring that the modified score $s'$ is (close to) uniformly distributed, leading to asymptotically valid conditional prediction sets, while keeping marginal validity.

An alternative strategy to achieve asymptotic conditional coverage is to define a finite partition $\mathcal{A} = \{ A_1, \ldots, A_K \}$ of the feature space $\mathcal{X}$, and construct prediction regions locally within each partition element. Concretely, let $T: \mathcal{X} \to \mathcal{A}$ be the function that maps each feature vector $\x$ to its corresponding region in $\mathcal{A}$. Then, divide the calibration set into subsets
$I_j = \{ i  : T(\X_i) = A_j \}, \ j = 1, \ldots, K.$ 
Finally, within each region $A_j$, compute the conformal quantile $t_{j,1-\alpha}$ as the $(1 + 1/n_j)(1 - \alpha)$ empirical quantile of scores $s_i$ for $i \in I_j$, where $n_j = |I_j|$. The local prediction region is then defined as
$$
C_{\text{local}}(\x) = \left\{ y \in \mathcal{Y} : s(y; \x) \leq t_{j,1-\alpha} \right\}, \quad \text{for } \x \in A_j.
\label{eq:splitinterval-local}
$$
This procedure guarantees
$
\P\left(Y_{n+1} \in C_{\text{local}}(\X_{n+1}) \mid \X_{n+1} \in A_j \right) \geq 1 - \alpha.
$
As the partition becomes finer (i.e., as $A_j$ shrinks), the method approaches conditional validity. Several strategies for defining such partitions have been proposed \citep{Vovk2012, Lei2014, bostrom2020mondrian, bostrom2021mondrian, Barber2021}. We use \texttt{LoCart} \citep{cabezas2025regression}, which constructs a regression tree to partition the feature space by predicting the conformity score $s(y; \x)$ from $\x$. The resulting partition approximates the coarsest one where the conditional distribution of $s \mid \x$ is constant \citep{meinshausen2006quantile, cabezas2025regression}, enabling local conformalization to closely match conditional coverage with a minimal number of data-efficient regions.
 
\subsection{Bayesian Credible Sets}

\label{sec:bayesian_credible}

Bayesian credible sets typically take the form
\begin{equation}
\label{eq:def-credible-set}
C(\mathbf{x}) = \{\theta : s(\theta; \mathbf{x}) \leq t_{1-\alpha}(\mathbf{x})\},
\end{equation}
where $s(\theta; \mathbf{x})$ is a scoring function computed using the posterior distribution, and the threshold $t_{1-\alpha}(\mathbf{x})$ is chosen to satisfy the conditional coverage condition:
\begin{equation*}
\mathbb{P}\left(\theta \in C(\mathbf{x}) \mid \mathbf{x} \right) =
\int \mathbb{I}\left( s(\theta; \mathbf{x}) \leq t_{1-\alpha}(\mathbf{x}) \right) p(\theta \mid \mathbf{x})\, d\theta = 1 - \alpha.
\end{equation*}
In other words, the posterior probability that the credible set contains the parameter value must be $1-\alpha$, where the miscoverage level $\alpha$ is defined beforehand.

Different choices of the scoring function $s$ lead to different types of credible sets. For example (see Figure \ref{fig:cred_sets} for an illustration):
\begin{itemize}
  \item \textbf{(HPD Regions)} If $s(\theta; \mathbf{x}) = -p(\theta \mid \mathbf{x})$, the resulting set corresponds to the highest posterior density (HPD) region;
  \item \textbf{(Symmetric Regions)} In the case $\theta \in \mathbb{R}$ and if $s(\theta; \mathbf{x}) = \dfrac{|\theta - \mathbb{E}[\theta \mid \mathbf{x}]|}{\sqrt{\mathrm{Var}[\theta \mid \mathbf{x}]}}$, the resulting set is a central region based on the posterior mean and variance (see \citealt{masserano2023simulator} for multivariate extensions);
  \item  \textbf{(Quantile-based Regions)}  In the case $\theta \in \mathbb{R}$ and if $s(\theta; \mathbf{x}) = \max\left\{ q_{\alpha_1}(\mathbf{x}) - \theta,\, \theta - q_{\alpha_2}(\mathbf{x}) \right\}$, where $q_\alpha(\mathbf{x})$ denotes the $\alpha$-quantile of the distribution $\theta \mid \mathbf{x}$ and $\alpha_2-\alpha_1=1-\alpha$, the resulting set is the quantile-based interval $(q_{\alpha_1}(\mathbf{x}), q_{\alpha_2}(\mathbf{x}))$. In this case, the threshold satisfies $t_{1-\alpha}(\mathbf{x}) = 0$ by construction.
\end{itemize}

\begin{figure}[ht]
    \centering
    \includegraphics[width=1\linewidth]{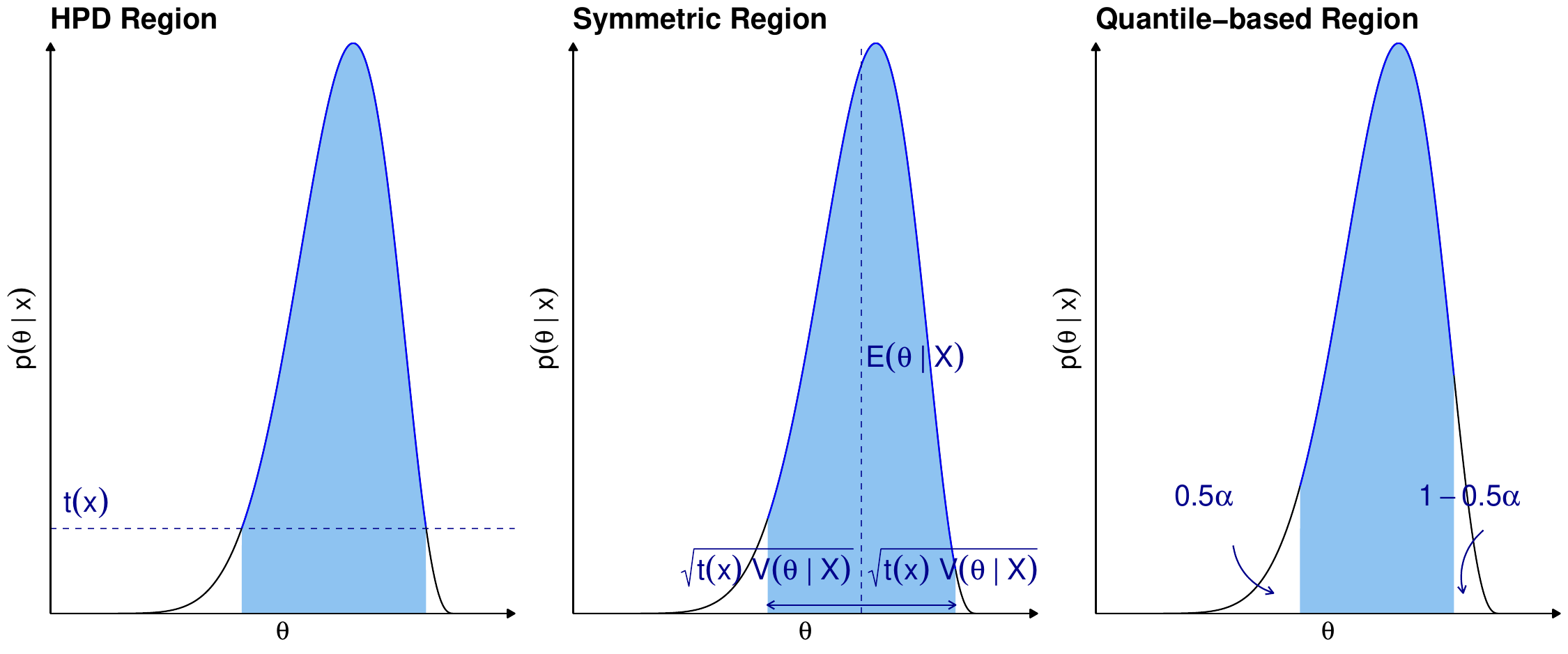}
    \caption{Credible regions for each distinct scoring function $s$.}
    \label{fig:cred_sets}
\end{figure}

When the true posterior $p(\theta \mid \mathbf{x})$ is available—either in closed form or through sampling—it is conceptually straightforward, albeit potentially computationally demanding, to compute a threshold $t_{1-\alpha}(\mathbf{x})$ that guarantees valid coverage as defined by the equation above.

Now suppose that an approximation $\widehat{p}(\theta \mid \mathbf{x})$ to the true posterior is used, for instance via neural posterior estimation, variational inference, or another method. 
The scoring function $s$ can then be derived from the approximate version of $p$. For instance, the score corresponding to HPD can be $-\widehat p(\theta\mid\x)$. However, in this setting the threshold $t_{1-\alpha}(\x)$  also needs to be estimated. A common approach is to choose its estimate $\widehat{t}_{1-\alpha}(\x)$ so that
\begin{align}
\label{eq:naive_approach}
\widehat{\mathbb{P}}\left(\theta \in \widehat{C}(\mathbf{x}) \mid \mathbf{x} \right) :=
\int \mathbb{I}\left( s(\theta; \mathbf{x}) \leq \widehat{t}_{1-\alpha}(\mathbf{x}) \right) \widehat{p}(\theta \mid \mathbf{x})\, d\theta = 1 - \alpha.
\end{align}

That is, $\widehat{t}_{1-\alpha}(\mathbf{x})$ is set so that the resulting credible set contains $\theta$ with posterior probability $1 - \alpha$ under the approximate $\widehat{p}$. However, if $\widehat{p}$ poorly approximates the true posterior $p$, the resulting coverage can deviate substantially from the nominal level. \ourmethod \ addresses this issue.

\section{Our approach: \ \ourmethod}
\label{sec:methodology}

The key insight of our approach is to reinterpret $s(\theta; \x)$ --- defined in Equation~\eqref{eq:def-credible-set} --- as a conformity score in the conformal inference framework. This allows us to construct credible sets with improved coverage properties, even when the posterior $p(\theta \mid \x)$ is poorly approximated.

Let $\{(\theta_1, \X_1), \ldots, (\theta_B, \X_B)\}$ be a calibration dataset drawn independently from the joint distribution $\pi(\theta)p(\x\mid\theta)$ and not used for estimating the posterior $\widehat{p}(\theta\mid\x)$. 
Concretely, each  $(\theta,\X)$ is drawn by first sampling $\theta$ from the prior distribution $\pi(\theta)$ and then $\X \mid\theta$ from the statistical model $p(\x\mid\theta)$ (which is the forward simulator in SBI).
As we are adopting a standard Bayesian framework, the prior $\pi(\theta)$ is assumed to be known and fixed during both training and calibration, and thus the guarantees we obtain are stated with respect to this prior. In practice, the choice of $\pi(\theta)$ should reflect domain knowledge, or the use of weakly informative priors when subjective prior information is limited.

A first naive conformal approach defines the prediction set as
$C(\x) = \{\theta \in \Theta : s(\theta; \x) \leq t_{1-\alpha}\},$
where $t_{1-\alpha}$ is the $(1+1/B)(1 - \alpha)$-quantile of the scores $\{s(\theta_b; \x_b)\}_{b=1}^B$, as detailed in Algorithm \ref{alg:vanilla_cp}. This ensures marginal coverage:
$\P\left(\theta \in C(\X)\right) \geq 1 - \alpha$,
but offers no conditional coverage guarantees, which are central in Bayesian settings. Even as $B \to \infty$, $t$ converges to the $(1 - \alpha)$-quantile of the marginal distribution of $s(\theta; \X)$, not the conditional distribution of $s(\theta; \x)$ for fixed $\x$. To address this, we introduce conformal methods designed for asymptotic conditional validity. We present two implementations of \ourmethod{}, though other strategies from Section~\ref{sec:conformal} also apply:


\vspace{2mm}
\noindent
\blocartourmethod:
We adopt the strategy of inducing a partition of the covariate space $\mathcal{X}$ by fitting a regression tree to predict $s(\theta; \x)$ from $\x$. Conformal calibration is then applied within the tree leaf containing $\x$, using only calibration scores from that leaf. This technique, introduced by \citet{cabezas2025regression} and known as \texttt{LoCart}, is motivated by the fact that the resulting partition (i.e., the terminal leaves) groups observations $\mathbf{x}$ that share an approximately identical conditional score distribution, $s(\theta; \mathbf{x})|\mathbf{x}$. As a result, the average coverage probability $\mathbb{P}(\theta \in {C}(\x) \mid \x \in A)$ closely approximates the desired local conditional coverage $\mathbb{P}(\theta \in {C}(\x) \mid \X = \x)$, since the partition captures local behavior. The theoretical properties and regularity assumptions underpinning this are detailed in Appendix~\ref{appendix:proofs}.

In our implementation of \texttt{LoCart}, we adopt an augmented version that enriches the feature space by including an estimate of the conditional variance of the conformity score, $\mathbb{V}[s(\theta; \X) \mid \X]$, as an additional feature. This augmentation, along with standard pruning techniques (detailed in Appendix~\ref{appendix:comp_details}) to ensure robust and well-populated partitions, allows for more informative partitions and improves the efficiency of the local procedure. The method is summarized in Algorithm \ref{alg:cp4sbi_locart} in Appendix \ref{appendix:algorithms} and illustrated in Figure~\ref{fig:locart_CP4SBI}.

A key implementation detail concerns the calibration data. Although guaranteeing marginal coverage theoretically requires splitting $\mathcal{D}_{\text{cal}}$ (for partitioning and cutoff computation), we omit this step in practice. We use the entire calibration set for both tasks, as this has minimal empirical impact on marginal coverage and improves practical performance by maximizing the sample size available for local quantile estimation.

\begin{figure}[h]
    \centering
    \includegraphics[width=0.9\columnwidth]{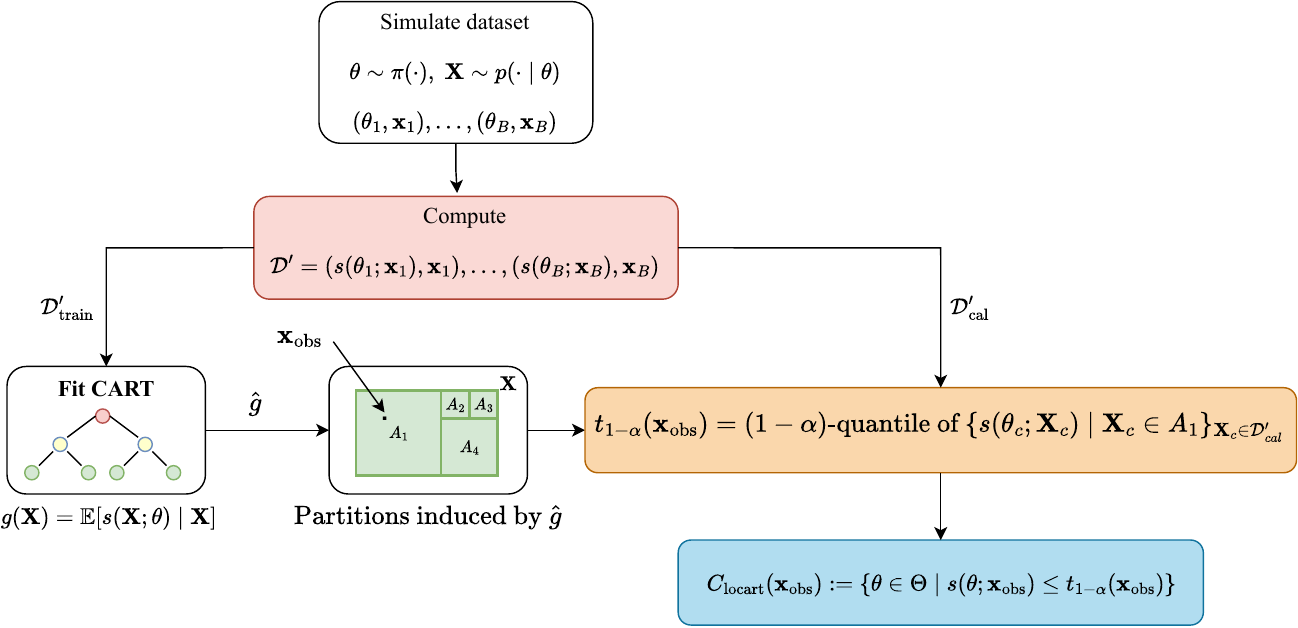}
    \caption{The \locartourmethod{} process begins with a simulated calibration dataset (top-center). From this, we compute conformity scores to build a new dataset $\mathcal{D}'$ (red), which is then split into partition and cutoff sets. A regression tree is trained on the partition set to predict scores $s$ from features $\mathbf{X}$ (CART panel), inducing a partition of the feature space (partition panel). For a new observation $\mathbf{x}_{\text{obs}}$, we find its corresponding region and compute a local cutoff as the $(1-\alpha)$ quantile of scores from the cutoff set that fall into that same region (orange). This local cutoff defines the final credible region (blue).}
    \label{fig:locart_CP4SBI}
\end{figure}

\vspace{2mm}
\noindent
\bcdfourmethod: 
This variant improves conditional coverage by transforming the score $s(\theta; \x)$ via an estimate $\widehat{F}(\cdot \mid \x)$ of its conditional cumulative distribution function, as proposed by \citealt{dheur2025multi}:
$
s'(\theta; \x) = \widehat{F}(s(\theta; \X) \mid \X = \x).
$
 In its original prediction setting, CDF-conformal needs the conditional distribution of $s(Y;\X) \mid \x$ to be approximated from scratch, which requires estimating a conditional density (via e.g. normalizing flows or a kernel density estimator) on top of the original conformal score. This results in high computational costs and a separate holdout set dedicated to learning such a distribution. Note, however, that in our simulator-based inference setting, an estimate of the posterior $\widehat{p}(\theta \mid \x)$ is \emph{already available} and we can use it to obtain $\widehat{F}(s \mid \x)$ with minimal additional computational cost.
In practice, we approximate $\widehat{F}(s(\theta;\x) \mid \x)$ using a Monte Carlo estimate based on posterior draws $\{\theta_j\}_{j=1}^M \sim \widehat{p}(\theta \mid \x)$, i.e., samples generated from the estimated posterior at $\x$. The approximation is given by the empirical CDF:
\begin{align}
\label{eq:ECDF_eq}
\widehat{F}_M(s(\theta;\x) \mid \x) = \frac{1}{M} \sum_{j=1}^M \mathbb{I}\big(s(\theta_j;\x) \leq s(\theta;\x)\big).
\end{align}

This corresponds to using the ECDF method from \citep{dheur2025multi}. We emphasize that these $M$ posterior draws are from the \textbf{already-trained} posterior estimator, \textbf{not} the expensive \textbf{simulator}. This step is therefore computationally inexpensive, yet leverages the approximator's information for more adaptive credible sets. \cdfourmethod{} is summarized in Algorithm \ref{alg:cp4sbi_cdf} in Appendix \ref{appendix:algorithms} and illustrated in Figure~\ref{fig:cdf_CP4SBI}. 

\begin{figure}[h]
    \centering
    \includegraphics[width=0.9\columnwidth]{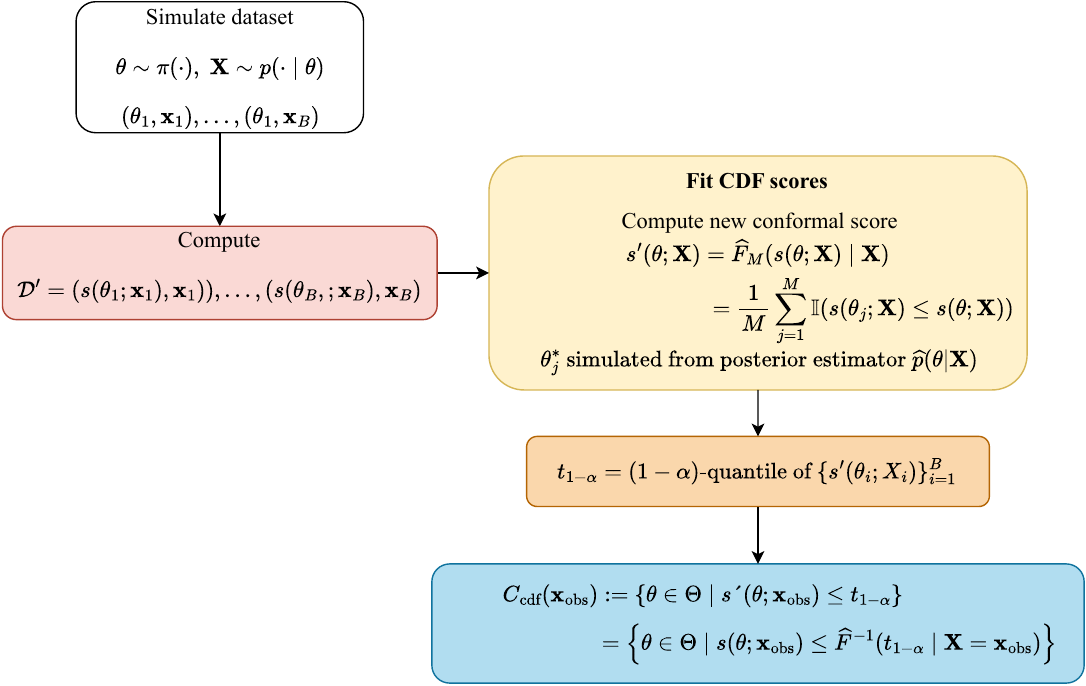}
    \caption{The \cdfourmethod{} process begins with a simulated calibration dataset (top-left), from which we compute nonconformity scores to build the dataset $\mathcal{D}'$ (red). A posterior estimator is then used to estimate the conditional CDF of the scores, transforming each score into its CDF value $s'(\X; \theta)$ (yellow). A single, global cutoff is computed from these newly transformed scores (orange). This global threshold defines the credible region for $\x_{\text{obs}}$, which translates back to a data-dependent, local cutoff in the original score space.}
    \label{fig:cdf_CP4SBI}
\end{figure}

\vspace{2mm}
\noindent
\textbf{Nuisance parameters and transformations of the parameter space.}
In many applications, one may only need credible sets for a subset of $\theta$ or for a transformation $\phi = g(\theta)$. Our method naturally accommodates these cases. Specifically, we approximate the posterior $p(\phi \mid \mathbf{x})$ using the same techniques as for $p(\theta \mid \mathbf{x})$, then compute scores for $\phi$ as described in Section~\ref{sec:bayesian_credible}. To calibrate the cutoff $t_{1-\alpha}(\mathbf{x})$, we use the transformed calibration set ${(\phi_1, \mathbf{X}_1), \ldots, (\phi_B, \mathbf{X}_B)}$ with $\phi_i = g(\theta_i)$. Figure~\ref{fig:credible_regions_comparison_2} shows our method in this setting, reducing a 10-dimensional problem (Gaussian Linear Uniform \citep{lueckmann2021benchmarking}) to the first two dimensions (details on the illustration in Appendix \ref{appendix:illustration_details}). Our approach yields regions close to the oracle with accurate coverage, while other methods tend to be overconfident.
\begin{figure}[h]
    \centering
    \includegraphics[width=1.0\columnwidth]{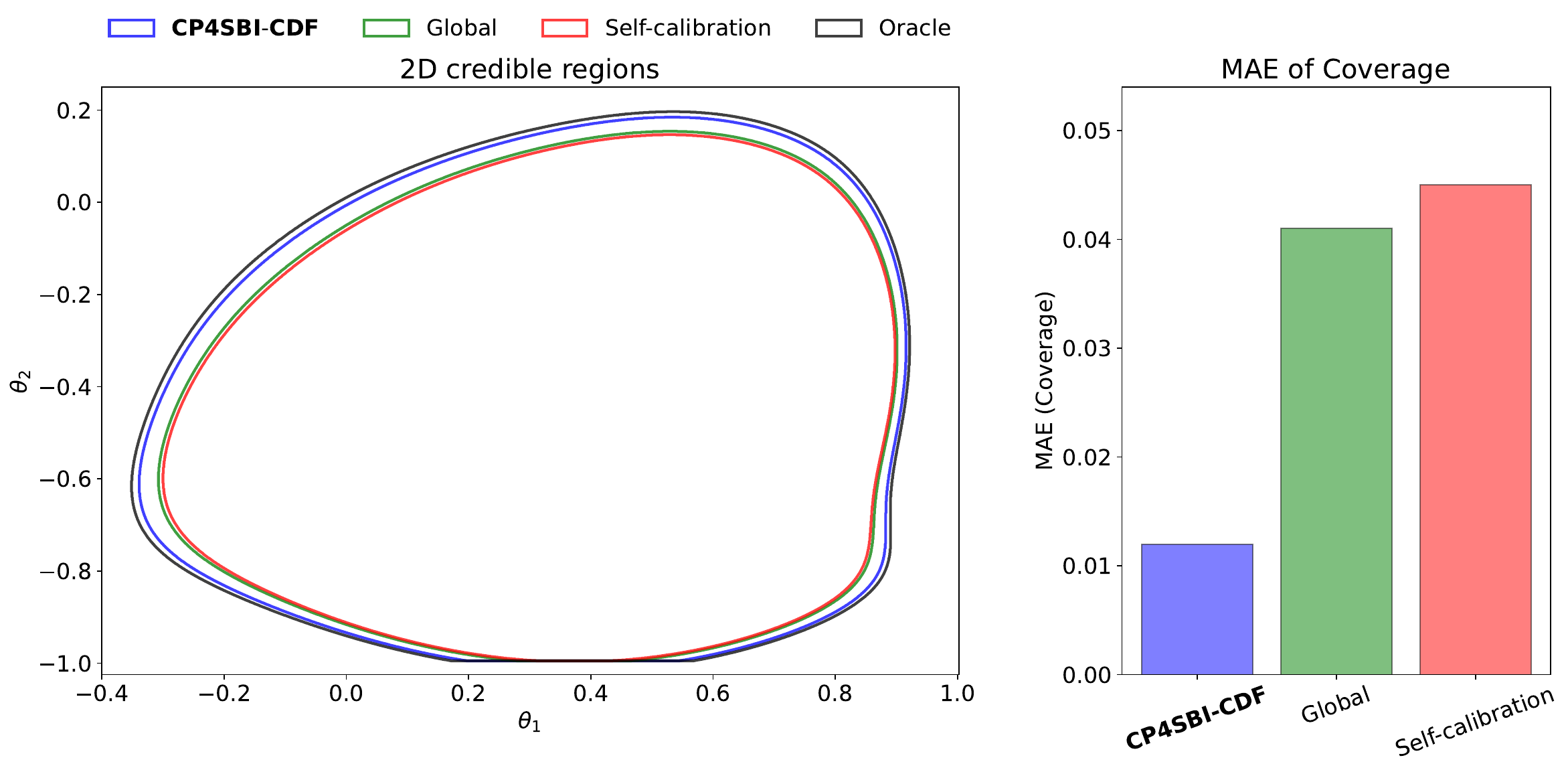}
    \caption{For a fixed observation $\x_{\text{obs}}$, we compare credible regions for the first two parameters of the 10-dimensional \textit{Gaussian Linear Uniform} benchmark. In this task, the simulator draws $\mathbf{x}$ from a Gaussian distribution with a fixed covariance and mean $\theta$, using a uniform prior over $\theta$. The results show our \cdfourmethod{} (blue) produces a region closer to the oracle (black) and exhibits a smaller deviation from nominal coverage.}
    \label{fig:credible_regions_comparison_2}
\end{figure}

\vspace{2mm}
\noindent
\textbf{Credible sets for continuous-flow generative models.}
 Many recent posterior estimation methods do not yield a closed-form expression for $\widehat{p}(\theta \mid \mathbf{x})$, or evaluating it is too costly. Instead, they provide independent samples $\theta_1, \ldots, \theta_L$ from $\widehat{p}(\theta \mid \mathbf{x})$ for each fixed $\mathbf{x}$—e.g., score-diffusion~\citep{linhart2024diffusion} and flow-matching models~\citep{wildberger2023flow}.
\ourmethod{} remains applicable in this setting. Given a scoring function $s(\theta;\mathbf{x})$, both \locartourmethod{} and \cdfourmethod{} can be used directly, as they only require a posterior sampler (see Algorithms \ref{alg:cp4sbi_locart} and \ref{alg:cp4sbi_cdf} in Appendix \ref{appendix:algorithms}). However, the standard scores from Section~\ref{sec:bayesian_credible} assume access to an explicit density and thus cannot be used as-is. The approximations below allow their use with posterior samples:
\begin{itemize}
\item The HPD score can be approximated as 
$s(\theta;\mathbf{x}) \propto - \sum_{l=1}^L K(\theta, \theta_l),$
where $K$ is a smoothing kernel, which corresponds to applying a kernel density estimator to the posterior samples.

\item The symmetric-region score can be approximated by
$s(\theta;\mathbf{x}) = \widehat{\mathbb{V}}^{-1/2}[\theta \mid \mathbf{x}] \cdot |\theta - \widehat{\mathbb{E}}[\theta \mid \mathbf{x}]|,$
where $\widehat{\mathbb{E}}$ and $\widehat{\mathbb{V}}$ denote the empirical mean and variance of the samples $\{\theta_1, \ldots, \theta_L\}$.

\item The quantile-based score can be approximated via
$s(\theta;\mathbf{x}) = \max\left\{ \widehat{q}_{\alpha_1}(\mathbf{x}) - \theta,\, \theta - \widehat{q}_{\alpha_2}(\mathbf{x}) \right\},$
where $\widehat{q}_{\alpha_1}(\mathbf{x})$ and $\widehat{q}_{\alpha_2}(\mathbf{x})$ are the empirical quantiles of the posterior samples.
\end{itemize}
Other scores from the multivariate conformal prediction literature can also be used in this framework, such as C-PCP \citep{wang2023probabilistic,dheur2025multi} and CP$^2$-PCP \citep{plassier2024conditionally}. In this work, we employ the Kernel Approximation to derive the HPD regions for these generative models. Algorithm \ref{alg:KDE_score} in Appendix \ref{appendix:algorithms} details the procedure for deriving these scores, and Figure \ref{fig:illustration_diffusion} illustrates the regions obtained by using our method in this setting on the Gaussian Mixture task \citep{lueckmann2021benchmarking} (details on the illustration can be found on Appendix \ref{appendix:illustration_details}). \ourmethod{} better approximates the oracle region and achieves coverage closer to the nominal rate.
\begin{figure}[h]
    \centering
    \includegraphics[width=1.0\linewidth]{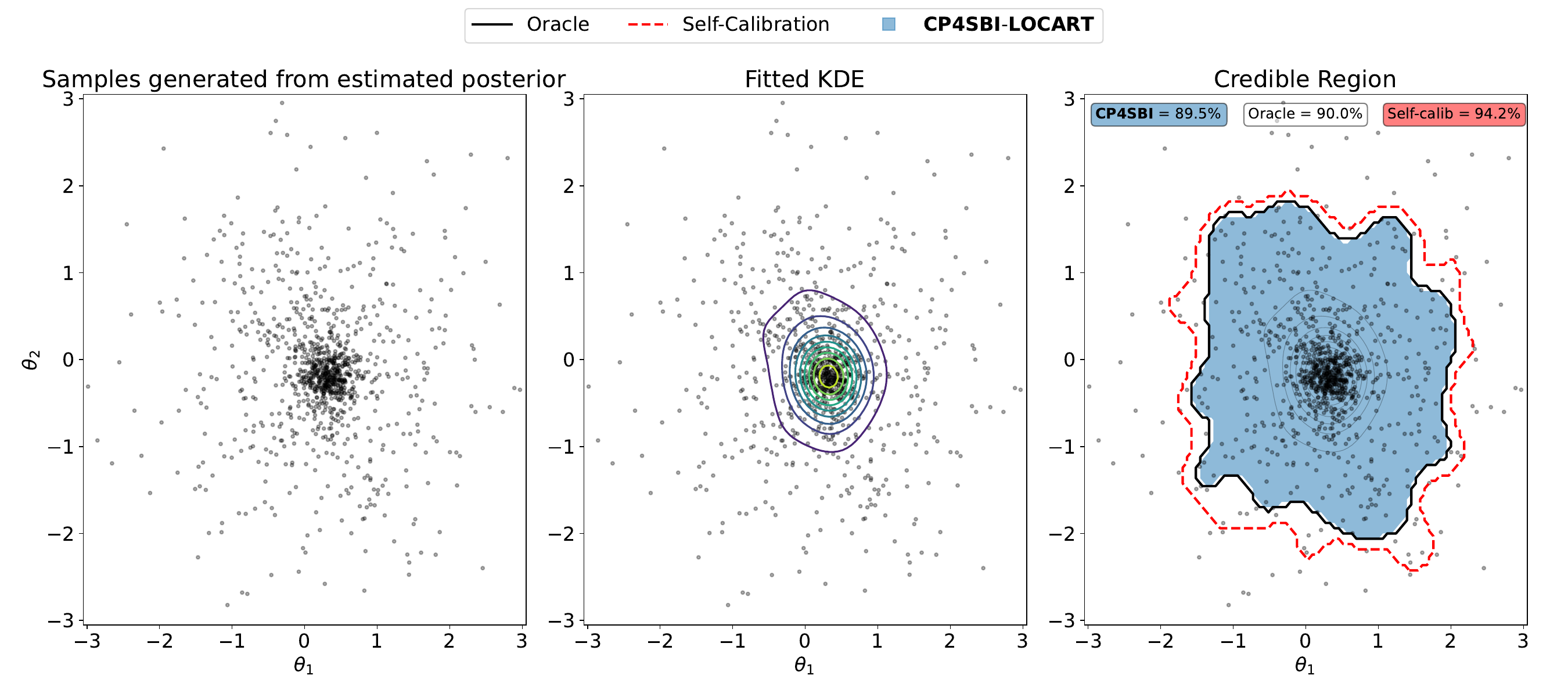}
    \caption{HPD region approximation for a continuous-flow generative model on the Gaussian Mixture benchmark, evaluated at a fixed $\mathbf{x}_{\text{obs}}$. The simulator draws $\mathbf{x}$ from a 2D Gaussian mixture with disparate covariances, where $\theta$ (uniform prior) defines the mean for both components. Our method achieves closer alignment to the oracle region and better coverage than self-calibration.}
    \label{fig:illustration_diffusion}
\end{figure}


\section{Theoretical guarantees}
\label{sec:theory}



We now present theoretical guarantees for the methods in Section~\ref{sec:methodology}. 
From this point on, we assume a disjoint data split: a training set $\{(\theta_i, \X_i)\}_{i=1}^K$ 
used to estimate the posterior $\widehat{p}(\theta \mid \x)$, and a calibration set 
$\{(\theta_i, \X_i)\}_{i=1}^B$ used to calibrate the conformal methods. 
All probabilities below are taken jointly over the randomness of the test pair $(\theta, \X)$ 
and of the calibration dataset $\mathcal{D}_{\text{cal}}$. 
Proofs are given in Appendix~\ref{appendix:proofs}.

\subsection{\locartourmethod}

The coverage guarantee of \locartourmethod{} follows from applying standard conformal prediction separately within each region $A_j$ of the partition it defines. Once the partition is fixed, the calibration scores \({s}(\theta; \x)\) remain exchangeable within each region, which allows us to apply conformal calibration independently in each subset. As a result, the quantile \({t}_{1-\alpha}(\x)\), computed using calibration points in \(A_j\), yields valid marginal coverage conditional on \(\x \in A_j\). 

\begin{theorem}[\locartourmethod{} local coverage]
\label{thm:locart_local_marginal}
Suppose the calibration pairs $\{(\theta_i, \X_i)\}_{i=1}^B$ and the test pair $(\theta,\X)$ are exchangeable.  
Let $\{A_j\}_{j=1}^J$ be the partition of $\mathcal{X}$ produced by the \locartourmethod{}.  
Denote by $A_j$ the cell of the partition containing $\X$, and let $t_{1-\alpha}(\X)$ be the $(1-\alpha)$ quantile of the conformity scores $s(\theta;\x)$ computed from the calibration pairs with $\X_i \in A_j$.  
Define the conformal set
\[
    C_{\text{locart}}(\x) := \bigl\{ \theta \in \Theta : s(\theta;\x) \leq t_{1-\alpha}(\x) \bigr\}.
\]
Then $C_{\text{locart}}(\x)$ satisfies both local and marginal coverage:
\[
    \mathbb{P}\!\big(\theta \in C_{\text{locart}}(\X) \,\big|\, \X \in A_j\big) \;\geq\; 1 - \alpha,
    \qquad
    \mathbb{P}\!\big(\theta \in C_{\text{locart}}(\X)\big) \;\geq\; 1 - \alpha.
\]
\end{theorem}



In addition, under the regularity conditions of \citet[Theorem~5]{cabezas2025regression}, \locartourmethod{} achieves asymptotic conditional coverage.

\begin{theorem}[\locartourmethod{} asymptotic conditional coverage]
\label{thm:locart_asymptotic}
Let $B$ denote the size of the calibration set. Under the assumptions stated in \citet[Theorem~5]{cabezas2025regression}, \locartourmethod{} satisfies
\[
\lim_{B \to \infty} \mathbb{P}\big(\theta \in C_{\text{locart}}(\X) \,\big|\, \X = \x\big) = 1 - \alpha,
\]
that is, it achieves conditional coverage in the limit as the calibration sample size grows.
\end{theorem}

\subsection{\cdfourmethod}


When using the transformed scores 
$s'(\theta; \x) = \widehat{F}_M(s(\theta; \x) \mid \x)$, the \cdfourmethod{} procedure reduces to applying standard conformal prediction with a modified nonconformity score. Here, $\widehat{F}_M(\cdot \mid \x)$ denotes the Monte Carlo estimate of the conditional CDF of the score $s(\theta; \x)$ given $\x$ based on $M$ samples drawn from the estimated posterior, defined by Equation~\ref{eq:ECDF_eq}. In the marginal setting, where calibration and test pairs $(\theta, \X)$ are exchangeable, this transformation preserves validity, and the method retains the usual marginal coverage guarantee.
\begin{theorem}[Marginal coverage]
\label{thm:cdf_marginal}
Assume the calibration pairs $\{(\theta_i, \X_i)\}_{i=1}^B$ and the test pair $(\theta, \X)$ are exchangeable, and let  
$s'_i = \widehat{F}_M(s(\theta_i; \X_i) \mid \X_i)$  
be the transformed conformity scores. Then the conformal region  
\[C_{\text{cdf}}(\x) := \{\theta \in \Theta : \widehat{F}_M(s(\theta; \x) \mid \x) \leq t_{1 - \alpha} \},\]  
with $t_{1 - \alpha}$ the empirical $(1 - \alpha)$-quantile of $\{s'_i\}_{i=1}^B$, satisfies  
\[\mathbb{P}(\theta \in C_{\text{cdf}}(\X)) \geq 1 - \alpha.\]
\end{theorem}

The conditional validity of \cdfourmethod{} builds on the probability integral transform, following the approach of \citet{dheur2025multi}. If the estimated posterior $\widehat{p}(\theta \mid \x)$ converges to the true posterior $p(\theta \mid \x)$ as the training size $K \to \infty$, and the number of posterior draws $M$ used to compute $\widehat{F}_M$ grows to infinity, then the transformed score
$
s'(\theta; \x) = \widehat{F}_M(s(\theta; \x) \mid \x)
$
converges to the true conditional CDF $F(s(\theta; \x) \mid \x)$ of the score. Since $\theta \sim p(\cdot \mid \x)$, this implies that $s'(\theta; \x)$ becomes approximately $\mathrm{Uniform}(0,1)$ conditional on $\x$. In this case, the calibration scores and the test score are approximately conditionally i.i.d., and the empirical quantile converges to the target level $1 - \alpha$, yielding asymptotic conditional coverage.

The asymptotic conditional validity of the procedure depends on the accuracy of the posterior approximation and of the CDF estimate:  $M$ must grow to ensure that $\widehat{F}_M$ converges to  $\widehat F$ implied by $\widehat p$, $B$ must grow so that the empirical quantile of the transformed scores converges to its population value, and $K$ must grow to guarantee that $\widehat{p}$ converges to $p$.
\begin{theorem}[\cdfourmethod{} asymptotic conditional coverage]
\label{thm:cdf_asymptotic}
Let $B$ be the calibration set size, $K$ the training set size used to estimate the posterior distribution $\widehat{p}(\theta \mid \x)$, and $M$ the number of posterior draws used to compute $\widehat{F}_M$. Under Assumption~\ref{ass:cdf_kl_convergence}, if $B \to \infty$, $K \to \infty$, and $M \to \infty$, then
\[
\lim_{B,\,K,\,M \,\to\, \infty} 
\mathbb{P}\big(\theta \in C_{\text{cdf}}(\X) \mid \X = \x\big) = 1 - \alpha.
\]
\end{theorem}

\section{Experiments}
\label{sec:experiments}

We compare our approach to baselines in terms of conditional and marginal statistical validity using ten SBI benchmarks introduced by \citealt{lueckmann2021benchmarking} (details in Table \ref{tab:benchmark_tasks}, in Appendix \ref{appendix:bench_details}).
These benchmarks are ideal for quantitative comparison because they provide access to samples from the true posterior, which is essential for computing coverage metrics. In contrast, in Subsection \ref{sec:real_application}, we illustrate our method's behavior on a genuinely complex SBI problem (a computational neuroscience model) where such ground-truth posteriors are unavailable. For this application, our analysis is necessarily qualitative, and we compare the resulting credible set shapes for a specific $\x_{\text{obs}}$.
All experiments have a nominal level at $1 - \alpha = 0.9$ with an overall simulation budget of $B_{\text{all}} = 10000$. From this budget, $80\%$ is allocated for training the posterior estimators, and the remaining $20\%$ is reserved for the calibration set. We consider other simulation budgets in Appendix \ref{appendix:additional_res}. We focus on constructing HPD regions, defining the conformity score as $s(\theta;\x)=- \widehat{p}(\theta\mid\x)$ or, for sample-based models, as $s(\theta;\x) \propto -\sum_{l = 1}^L K(\theta,\theta_l)$, using Scott's rule to determine kernel bandwidth \citep{scott2015multivariate}, where $\theta_l \sim \widehat{p}(\cdot\mid\x)$. Posterior estimation is carried out using conditional normalizing flows (NPE; \citep{greenberg2019automatic, papamakarios2021normalizing}) and a conditional diffusion model (NPSE; \citep{geffner2023compositional}). For both estimators, we use the implementations provided in the \texttt{sbi} package \citep{BoeltsDeistler_sbi_2025}, using their standard architectures.

\subsection{Metrics for conditional and marginal validity}
To assess conditional validity, we estimate the conditional coverage for each $\x$ as
\begin{align*}
    \delta(\x, C) = \frac{1}{K} \sum_{k = 1}^K \I\left(\theta_k \in C \left( \x \right) \right) \; ,
\end{align*}
 where $\theta_k \sim p(\theta \mid \mathbf{x})$. We then compute the Mean Absolute Error (MAE) over $B_{\text{sim}}$ fixed observations:
\begin{equation}  
\text{MAE} = \frac{1}{B_{\text{sim}}} \sum_{i=1}^{B_{\text{sim}}} \left| \delta ( \x_i,C ) - (1 - \alpha) \right| \; . 
\label{eq:mae}
\end{equation}  
Lower MAE indicates conditional coverage closer to the nominal level, reflecting better calibration.

We assess marginal coverage by simulating an independent test set $\mathcal{D}_{\text{test}} = \{(\theta_j, \x_j)\}_{j=1}^{B_{\text{test}}}$ and compute the Average Marginal Coverage (AMC) as:
\begin{align*}
    \text{AMC} = \frac{1}{B_{\text{test}}} \sum_{i = 1}^{B_{\text{test}}} \I \left(\theta_i \in C(\x_i)\right) \; .
\end{align*}
Values near the nominal level $1 - \alpha$ indicate good average coverage. We fix $K = 1000$ and $B_{\text{test}} = 2000$. For MAE, we use $B_{\text{sim}} = 500$ or 10, depending on the simulator's posterior sample availability in \texttt{sbibm} \citep{lueckmann2021benchmarking} (see Appendix \ref{appendix:bench_details} for details).

\subsection{Baselines}
\label{sec:baselines}

We compare our approach to three established methods for constructing credible regions:
\begin{itemize}
\item \textbf{Self-calibration}: 
This method constructs credible regions directly from the estimated posterior distribution, $\widehat{p}(\theta \mid \x)$. For each fixed data point $\x$, the method draws a set of $B_{\text{self}}$ (fixed at $1000$) posterior samples, $\theta_i \sim \widehat{p}(\cdot\mid\x)$, and evaluates their corresponding scores, $s(\theta_i;\x)$. From these scores, an empirical threshold $\widehat{t}(\x)$ is computed via Monte Carlo integration, following Equation~\eqref{eq:naive_approach}, to ensure the desired coverage level of $1-\alpha$. 


\item \textbf{Global} \citep{patel2023variational}: This method is a direct application of the vanilla conformal approach for constructing credible regions. First, it uses the calibration set $\mathcal{D}_{\text{cal}}$ to compute nonconformity scores for each data point. Then, it determines a single threshold, $t_{\alpha}$, from these scores, which is then applied uniformly to all new observations, $\x$. 

\item \textbf{HDR} \citep{chung2024sampling}: This multivariate recalibration method corrects miscalibration by learning a monotonic mapping $R$ from posterior density values using a calibration set $\mathcal{D}_{\text{cal}}$. For a new $\x$, it resamples from the posterior to align with $R$, producing calibrated samples that account for dependencies across dimensions. Cutoffs for HDRs are then computed using these recalibrated samples, similarly to self-calibration.
\end{itemize}

\subsection{Results and discussion}
Next, we evaluate the credible regions produced by different methods on the selected benchmarks. Focusing on the NPE base estimator, Figure \ref{fig:NPE_heatmap_figure_MAE_and_Marginal} shows that our approach significantly outperforms existing methods in terms of conditional coverage. Specifically, both \locartourmethod{} and \cdfourmethod{} methods perform statistically better in 8 out of 10 benchmarks, all while maintaining marginal coverage close to the nominal level. Overall, our methods perform well in almost all benchmarks, except for the Lotka-Volterra simulator, where only the global method shows a better performance. The right panel further shows that our proposed methods consistently maintain near-nominal marginal coverage, whereas approaches such as self-calibration and HDR lack this property on some benchmarks. The ability of our methods to produce efficient and well-calibrated regions across various benchmarks, parameter spaces, and data distributions highlights their robust performance.
\begin{figure}[ht]
    \centering
    \includegraphics[width=1\columnwidth]{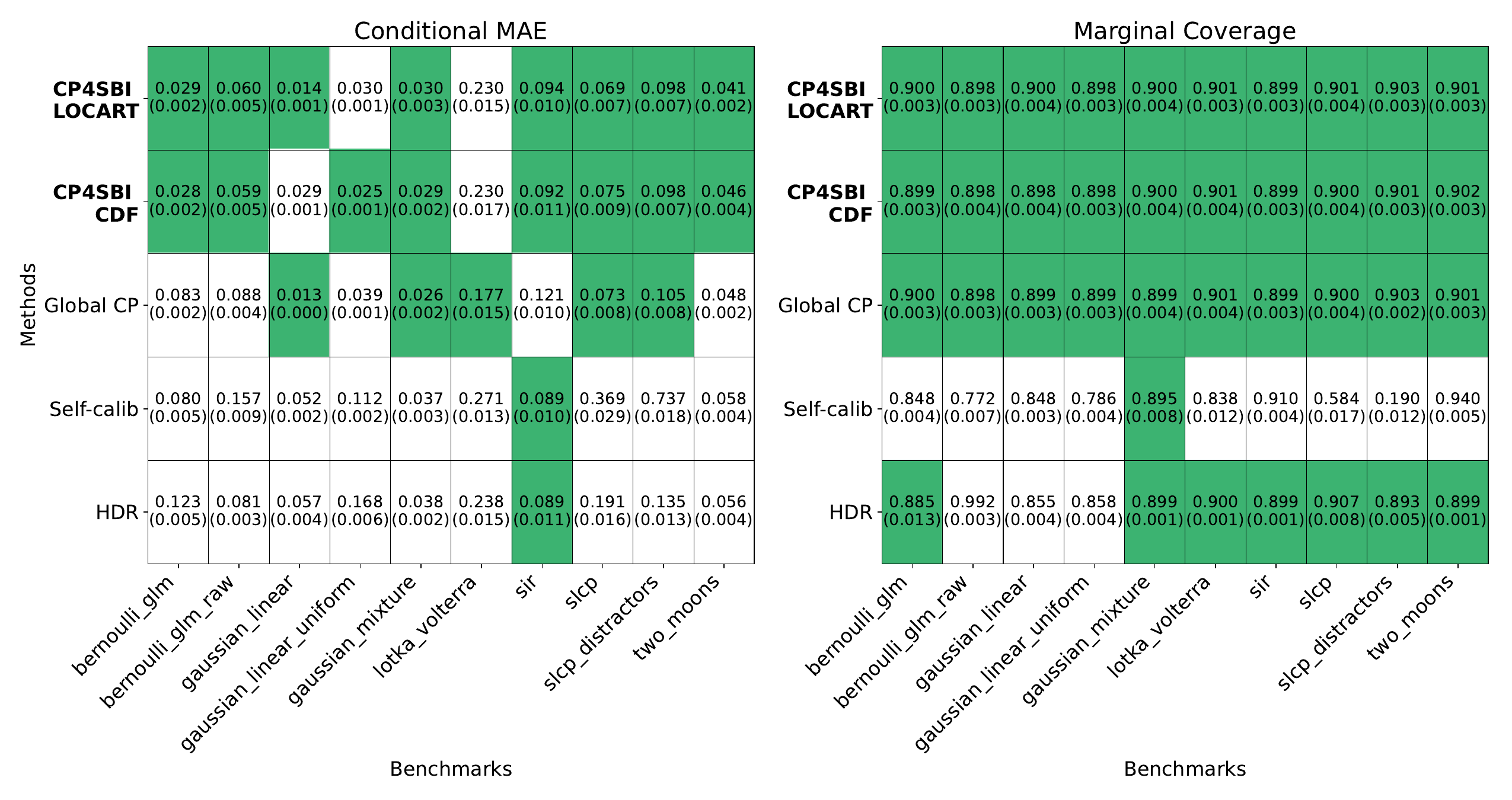}
    \caption{Conditional MAE (left) and marginal coverage (right) for NPE-based credible sets are shown across benchmark tasks, with means and 95\% confidence intervals over 50 runs. Green cells indicate statistically significant improvements (lower MAE or coverage near the nominal level). Our methods, \locartourmethod{} and \cdfourmethod{}, outperform others in 8 of 10 tasks (left) and maintain close-to-nominal marginal coverage (right), demonstrating improved posterior calibration across varied tasks.}
    \label{fig:NPE_heatmap_figure_MAE_and_Marginal}
\end{figure}

Figure \ref{fig:NPSE_heatmap_figure_MAE_and_Marginal} shows that our strategy for constructing calibrated HPD regions from an NPSE base estimator (i.e. a score diffusion model) is also successful. Both of our methods demonstrate better conditional coverage performance than other approaches, as well as solid marginal coverage. In particular, \locartourmethod{} proved to be the most effective, producing credible regions that simultaneously improve conditional coverage while respecting marginal coverage. While the HDR method also shows acceptable conditional coverage in some settings, it does not achieve marginal coverage in any dataset. This highlights the efficiency of our approach in recalibrating credible regions across different posterior estimators.
\begin{figure}[ht]
    \centering
    \includegraphics[width=1\columnwidth]{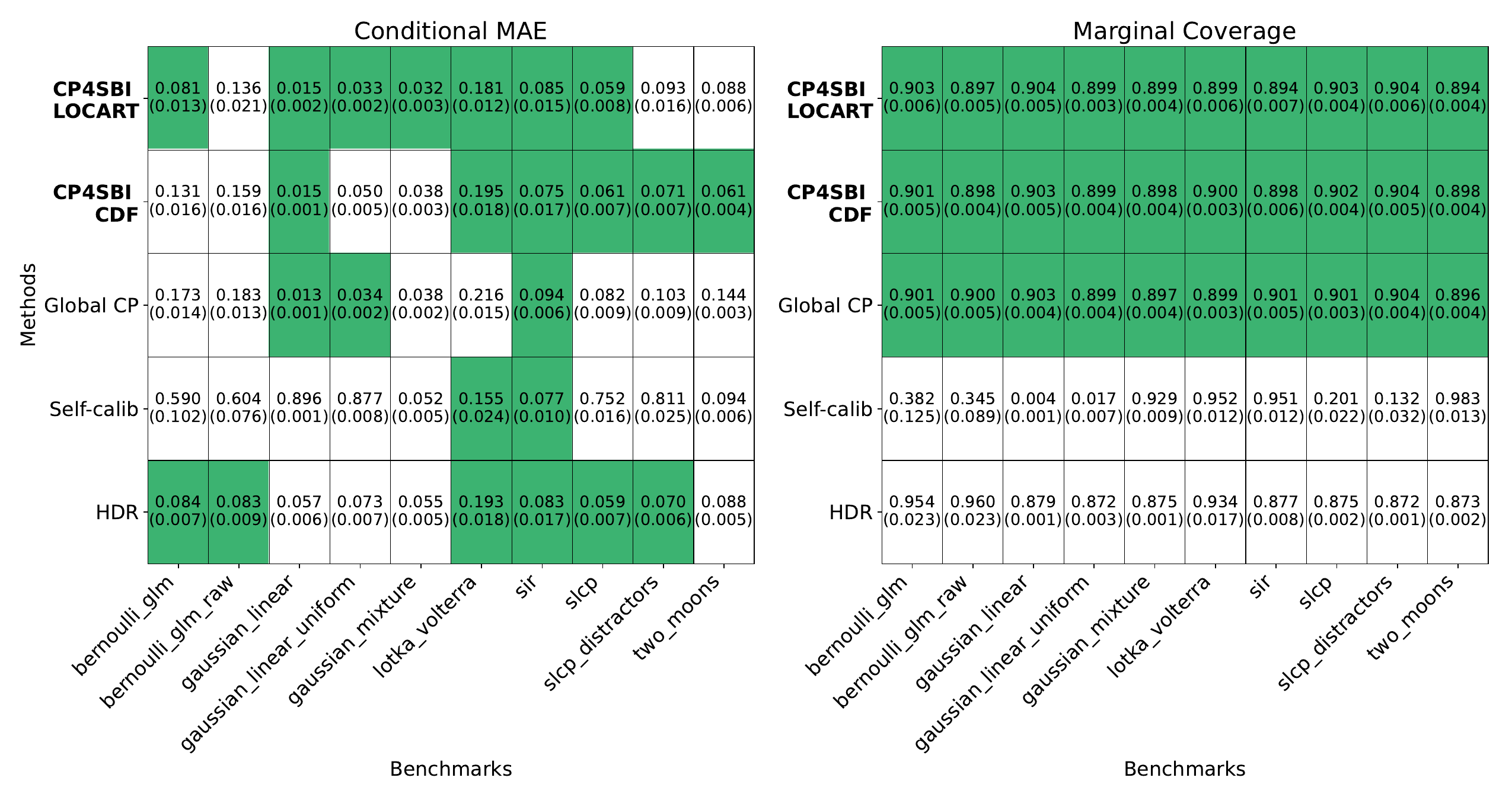}
    \caption{Conditional MAE (left) and marginal coverage (right) for NPSE-based credible sets across benchmark tasks, averaged over 15 runs with 95\% confidence intervals. Green cells highlight statistically significant coverage near the nominal level. \locartourmethod{} ranks among the best in 7 of 10 tasks. While HDR shows good conditional coverage, it fails to maintain marginal coverage. Our methods perform well on both metrics, ensuring calibrated inference.}
    \label{fig:NPSE_heatmap_figure_MAE_and_Marginal}
\end{figure}
Figure \ref{fig:NPE_heatmap_figure_MAE_other_budgets} in Appendix \ref{appendix:additional_res} shows that our approach consistently displays improved conditional coverage performance across two additional budgets ($B_{\text{all}} = 2000$ and $B_{\text{all}} = 20000$). A caveat exists, however: for the smaller budget, \locartourmethod{} performs comparably to the global one, likely because its partitioning strategy struggles with sparse calibration data. For this case, \cdfourmethod{} shows better results. This illustrates how \ourmethod{} adapts its performance based on the available data, highlighting its flexibility across different calibration budgets and posterior estimators of varying quality. This adaptiveness is also reflected in the size results in Table \ref{tab:hpd_areas_2d} in Appendix \ref{appendix:additional_res}, which show \ourmethod{}'s ability to adjust region size, unlike the more rigid global or non-conformal methods. All results obtained in this section reinforce the capacity of \ourmethod{} to enhance calibration in credible regions for fixed observations $\x$.

\subsection{An example on computational neuroscience}
\label{sec:real_application}
A neural mass model (NMM) is a complex non-linear model from computational neuroscience consisting of a system of stochastic differential equations describing the generation of neural activity on a cortical column~\citep{jansen1995electroencephalogram}. Specifically, we use the implementation detailed in~\cite{buckwar2020spectral}, consisting of a three-dimensional parametrization. In this model, $\theta_1$ represents the degree of connectivity between excitatory and inhibitory neurons. The other parameters, $\theta_2$ and $\theta_3$, jointly model the statistical properties (mean and variance) of incoming oscillations from neighboring cortical columns. These three parameters govern the model's dynamics, which in turn generate the output data $\mathbf{x}$—a time series representing the aggregate electrical signals, similar to EEG recordings \cite{rodrigues2021hnpe}.

We evaluate the posterior credible regions by examining each parameter pair of parameters individually, treating the remaining parameter as a nuisance variable. To this end, we train dedicated NPE approximators for each pair and compute their corresponding credible regions, using a training budget of $10000$. Following standard practice in SBI benchmarks, we also analyze a specific observation $\mathbf{x}_{\text{obs}}$ generated from fixed values of the parameters of interest. We compare our approach \locartourmethod{} with the global and the self-calibrated one. Figure \ref{fig:credible_regions_comparisons_pairs} illustrates the resulting $90\%$ HPD regions for this simulator, obtained under a calibration budget of $B = 4000$.
\begin{figure}[htb]
    \centering
    \includegraphics[width=1\columnwidth]{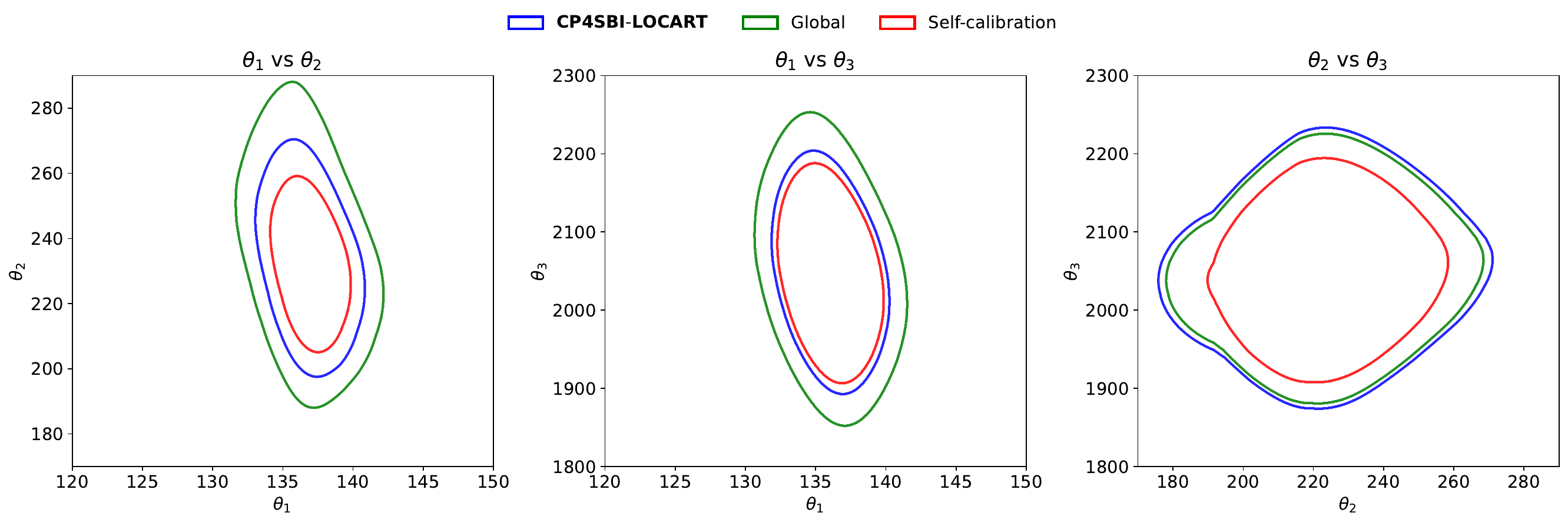}
    \caption{Comparison of $90\%$ HPD credible regions for the NMM model. \locartourmethod{} (in blue) demonstrates clear adaptivity, yielding large, medium, and small credible regions depending on the parameter combination. This contrasts sharply with the Global method (in green), which only produces large regions, and self-calibration (in red), which only produces small ones.}
    \label{fig:credible_regions_comparisons_pairs}
\end{figure}

We observe that \locartourmethod{} (blue) displays highly adaptive behavior across the different parameter pairs, yielding differently sized regions (medium, small, and large, respectively) for each combination. This contrasts sharply with the baselines: the Global method (green) consistently produces large regions in all three scenarios, while Self-calibration (red) consistently produces small regions. This result highlights how \ourmethod{} successfully enhances the local representation of HPD regions. It provides credible sets that are clearly tailored to the specific parameter pair, rather than defaulting to a single, uniform behavior. This demonstrates the method's ability to adapt to local structures, all while retaining the formal guarantee of marginal coverage.


\section{Final Remarks}
\label{sec:conclusions}
In this paper, we tackled the fundamental problem of miscalibration in simulation-based inference. 
We proposed \ourmethod{}, a post-hoc, model-agnostic framework using conformal prediction to build credible sets with finite-sample guarantees. It applies to any SBI method offering posterior samples or density estimates and supports any scoring function, enabling calibrated HPD regions, symmetric intervals, or custom sets.

We proposed two calibration strategies. The first, \cdfourmethod{}, achieves asymptotic conditional coverage by recalibrating scores using an estimate of their conditional CDF. Its guarantees strengthen as the underlying posterior approximation improves—typically as the training sample size increases. The second, \locartourmethod{}, provides finite-sample local coverage by partitioning the data space with a regression tree, adapting the calibration to regions of varying difficulty. This method also attains asymptotic conditional coverage as the size of the calibration set grows. Both methods guarantee correct marginal coverage.
 Our experiments on established SBI benchmarks confirmed that \ourmethod{} effectively improves uncertainty quantification for both neural posterior estimators and diffusion-based models.
 Although we focused on these two local conformal approaches, other modern methods for targeting conditional coverage—such as the one proposed by \cite{gibbs2025conformal} could also be explored in the SBI setting. Investigating their performance in this context is a promising direction for future work.

Despite offering stronger coverage guarantees, \ourmethod{} has some limitations. The tightness of its credible regions depends on the quality of the initial posterior; poor approximations lead to wider, albeit calibrated, sets. Additionally, \locartourmethod{} requires a sufficiently large calibration set to populate its partitions, which can be challenging in high-dimensional tasks. This can be partly mitigated by including posterior variance estimates or other summary statistics in the partitioning features.

This work opens several avenues for future research. Although we focused on two conformal methods, other techniques may yield better performance in specific settings (see, e.g., \citealt{plassier2025rectifying}). We also plan to extend the \ourmethod{} framework to more challenging inference scenarios, such as those involving hierarchical models or discrete parameters. Finally, its applicability is not limited to SBI; it could be a valuable tool for calibrating posteriors from other methods like Variational Inference or Approximate Bayesian Computation.

In conclusion, \ourmethod{} offers a flexible and theoretically grounded approach to improving uncertainty quantification in computational science. By producing credible sets with stronger local and conditional coverage, it enhances the reliability of inferences from complex simulator models. This marks a meaningful advance for many scientific fields, providing a general-purpose tool to support trustworthy discovery.

Code to implement \ourmethod{} and reproduce the experiments and illustrations is available at \url{https://github.com/Monoxido45/CP4SBI}.

\section*{Acknowledgements}
This study was financed in part by the Coordenação de Aperfeiçoamento de Pessoal de Nível Superior - Brasil (CAPES)
- Finance Code 001. L.M.C.C is grateful for the fellowship
provided by São Paulo Research Foundation (FAPESP),
grants 2022/08579-7 and 2025/06168-8. PLCR was supported by a national grant managed by the French National Research Agency (Agence Nationale de la Recherche) attributed to the SBI4C project of the MIAI AI Cluster, under the reference ANR-23-IACL-0006. RI is grateful for the financial support of FAPESP (grant 2023/07068-1) and CNPq (grants 422705/2021-7 and 305065/2023-8 and 403458/2025-0)



\bibliographystyle{plainnat}
\bibliography{biblio}

\newpage
\appendix
\section*{Appendix}
\section{Algorithms}
This appendix details all algorithms related to our approach. Algorithms \ref{alg:cp4sbi_locart} and \ref{alg:cp4sbi_cdf} describe the procedures for \locartourmethod{} and \cdfourmethod{}, respectively. Algorithm \ref{alg:KDE_score} further details how we approximate Highest Posterior Density (HPD) scores for implicit generative methods by using Kernel Density Estimation.
\label{appendix:algorithms}
\begin{algorithm}[ht]
\caption{CP4SBI - LoCart: Local Conformal Prediction for SBI}
\label{alg:cp4sbi_locart}
\KwIn{
    Calibration set $\mathcal{D} = \{(\theta_i, \x_i)\}_{i=1}^B$,
    posterior approximator $\widehat{p}(\theta \mid \x)$,
    conformity score $s(\theta; \X)$,
    nominal level $\alpha$,
    new observation $\x_{\text{obs}}$
}

\textbf{Step I: Score computation} \\
\Indp
1: For each $(\theta_i, \x_i) \in \mathcal{D}$, compute $s_i = s(\theta_i; \x_i)$ using $\widehat{p}$. \\
2: Form the scored dataset $\mathcal{D}' = \{(s_i, \x_i)\}_{i = 1}^B$, then randomly split it into partition $\mathcal{D}'_{\text{part}}$ and cutoff $\mathcal{D}'_{\text{cut}}$ subsets. \\
\Indm

\textbf{Step II: Partition learning} \\
\Indp
1: Fit a regression tree on $\mathcal{D}'_{\text{part}}$ to predict $s$ from $\x$. \\
2: Use the tree to induce a partition $\mathcal{A} = \{A_1, \dots, A_K\}$ of the feature space. \\
3: Define a region mapping $T: \mathcal{X} \rightarrow \mathcal{A}$ such that $T(\x) = A_j$ if $\x \in A_j$. \\
\Indm

\textbf{Step III: Local quantile estimation} \\
\Indp
1: For each region $A_j \in \mathcal{A}$, define the set of calibration indices $I_j = \{i \mid \x_i \in A_j,\; (s_i, \x_i) \in \mathcal{D}'_{\text{cut}}\}$. \\
2: Compute the local cutoff $t_j$ as the empirical $(1 + 1/|I_j|)(1 - \alpha)$-quantile of $\{s_i\}_{i \in I_j}$. \\
\Indm

\textbf{Step IV: Credible region construction} \\
\Indp
1: Assign $\x_{\text{obs}}$ to region $A_k = T(\x_{\text{obs}})$. \\
2: Return the credible region:
\[
R_{\text{CP4SBI}}(\x_{\text{obs}}) = \left\{ \theta \mid s(\theta; \x_{\text{obs}}) \leq t_k \right\}
\]
\Indm

\KwOut{Credible region $R_{\text{CP4SBI}}(\x_{\text{obs}})$ with marginal and local $1 - \alpha$ coverage}
\end{algorithm}

\begin{algorithm}[H]
\caption{CP4SBI-CDF: CDF-split for SBI}
\label{alg:cp4sbi_cdf}
\KwIn{Calibration set $\mathcal{D} = \{(\theta_i,\x_i)\}_{i =1}^B$, estimated posterior $\widehat{p}(\theta\mid \x)$, conformity score $s(\theta;\x)$, nominal level $\alpha$, new observation $\x_{\text{obs}}$}

\textbf{Step I: Score computation} \\
\Indp
1: For each $(\theta_i, \x_i) \in \mathcal{D}$, compute $s_i = s(\theta_i; \x_i)$ using $\widehat{p}$. \\
2: Form the scored dataset $\mathcal{D}' = \{(s_i, \x_i)\}_{i = 1}^B$ \\
\Indm

\textbf{Step II: Compute CDF scores} \\
\Indp
1: Compute $s'_i = s'(\theta_i; \x_i) = \widehat{F}(s(\theta_i;\X)\mid \mathbf{X} = \x_i)$ on $\mathcal{D}'$, using $\widehat{p}$ to estimate each $\widehat{F}(\cdot \mid \mathbf{X} = \x_i)$. \\
\Indm
\textbf{Step III: Cutoff estimation on transformed scores} \\
\Indp
1: Compute the cutoff $t_{1-\alpha}$ as the empirical $(1 + 1/B)(1 - \alpha)$-quantile of $\{s_i'\}_{i = 1}^B$ \\
\Indm

\textbf{Step IV: Compute credibility set} \\
\Indp
1: Compute the set ${C}_{\mathrm{CP4SBI}}(\x_{n + 1})$ as:
\vspace{-2.5mm}
   \begin{align*}
   {R}_{\mathrm{CP4SBI}}(&\x_{\text{obs}}) = \{\theta \mid s'(\theta; \x_{\text{obs}}) \leq t_{1-\alpha} \} \\
   &= \{\theta \mid s(\theta; \x_{\text{obs}}) \leq \widehat{F}^{-1}(t_{1 - \alpha}\mid\X = \x_{\text{obs}}) \} \; .
   \end{align*}
\\
\Indm
\KwOut{Credible region $R_{\text{CP4SBI}}(\x_{\text{obs}})$ with marginal and asymptotic conditional $1 - \alpha$ coverage}
\end{algorithm}

\bigskip

\begin{algorithm}[H]
\caption{KDE approximation of HPD score}
\label{alg:KDE_score}
\KwIn{Calibration set $\mathcal{D} = \{(\theta_i,\x_i)\}_{i =1}^B$, estimated posterior $\widehat{p}(\theta\mid \x)$, KDE budget $L$}
\Indp
For each $(\theta, \x) \in \mathcal{D}$, do: \\
\hspace{5mm} Simulate $L$ samples $\tilde{\theta} \sim \widehat{p}(\cdot \mid \x)$ \\
\hspace{5mm} Fit a KDE on data generated from posterior estimator $\{\tilde{\theta}_l\}_{l= 1}^L$\\
\hspace{5mm} Define $s(\theta; \x) = \frac{1}{L} \sum_{l = 1}^L K(\theta,\theta_l)$
\\
\Indm
\KwOut{Conformal scores $\{s(\theta_i, \x_i)\}_{i = 1}^B$}
\end{algorithm}

\begin{algorithm}[ht]
\caption{Vanilla Conformal Prediction applied to SBI}
\label{alg:vanilla_cp}
\KwIn{
    Calibration set $\mathcal{D} = \{(\theta_i, \x_i)\}_{i=1}^B$,
    posterior approximator $\widehat{p}(\theta \mid \x)$,
    conformity score $s(\theta; \X)$,
    nominal level $\alpha$,
    new observation $\x_{\text{obs}}$
}

\textbf{Step I: Score computation} \\
\Indp
1: For each $(\theta_i, \x_i) \in \mathcal{D}$, compute scores $s_i = s(\theta_i; \x_i)$ using $\widehat{p}$. \\
\Indm

\textbf{Step II: Cutoff estimation} \\
\Indp
1: Compute the cutoff $t_{1-\alpha}$ as the empirical $(1 + 1/B)(1-\alpha)$-quantile of $\{s_i\}_{i = 1}^B$ \\
\Indm

\textbf{Step III: Credible region construction} \\
\Indp
1: Compute the set $R_{\text{global}}(\x_{n+1})$ as:
\[
R_{\text{global}}(\x_{\text{obs}}) = \left\{ \theta \mid s(\theta; \x_{\text{obs}}) \leq t_{1 - \alpha} \right\}
\]
\Indm

\KwOut{Credible region $R_{\text{global}}(\x_{\text{obs}})$ with marginal $1 - \alpha$ coverage}
\end{algorithm}

\clearpage
\section{Proofs}\label{appendix:proofs}

This appendix contains formal proofs for the theoretical results stated in Section~\ref{sec:theory}. We organize the material according to the corresponding methods: \locartourmethod{} and \cdfourmethod{}.

\subsection{\locartourmethod}

We begin by justifying the finite-sample local and marginal coverage guarantees of \locartourmethod{}. The result follows directly from standard conformal prediction arguments applied independently within each partition element.

\begin{proof}[Proof of Theorem \ref{thm:locart_local_marginal}]
    This is a straightforward application of standard conformal prediction (see, e.g., \citealt{angelopoulos2023conformal}). Since the conformal threshold \({t}_{1-\alpha}(\x)\) is computed within each region \(A_j\) using calibration scores that are exchangeable within that region, the conformal guarantee holds conditional on \(\x \in A_j\). Marginal validity follows by averaging over all regions via the law of total probability. A complete version of this argument is presented in \citet[Theorem~2]{cabezas2025regression}.
\end{proof}

We now present the proof of the asymptotic conditional coverage guarantee for \locartourmethod{}. This result builds upon the regularity assumptions established in \cite{cabezas2025regression}.

\begin{assumption}
\label{assumption:partition_size}
Let $A_B(\x)$ denote the \locartourmethod{} partition element assigned to $\x$, when the calibration set has size $B$. Then,
\[
    \liminf_{B \to \infty} \; B \cdot \inf_{\x \in \mathcal{X}} \P[A_B(\x)] \;>\; 0.
\]    
\end{assumption}

\begin{assumption}
\label{assumption:quantile_convergence}
Let $F_{\x}(t \mid \X_i) := \P[s(\theta_i; \X_i) \le t \mid \X_i \in A_B(\x)]$ denote the conditional distribution of the conformity scores within region $A_B(\x)$.  
For every $\x$, $F_{\x}$ is continuous and increasing in a neighborhood of the $(1 - \alpha)$ quantile of the true conditional score distribution $F(t \mid \x)$, and
\[
    \sup_{\x \in \mathcal{X}} 
    \big| F_{\x}^{-1}(1 - \alpha \mid \X_i) - F^{-1}(1 - \alpha \mid \x) \big|
    = o_{\P}(1).
\]
\end{assumption}

\begin{proof}[Proof of Theorem \ref{thm:locart_asymptotic}]
Under Assumptions \ref{assumption:partition_size} and \ref{assumption:quantile_convergence}, the result follows directly from \citet[Theorem~5]{cabezas2025regression}.
In particular, Assumption \ref{assumption:partition_size} ensures that the partition elements $A_B(\x)$ remain sufficiently populated as $B$ increases,
while Assumption \ref{assumption:quantile_convergence} guarantees that the estimated conditional quantiles of the conformity scores converge in probability
to their true conditional counterparts within each region.
\end{proof}

\subsection{\cdfourmethod}

We begin by formalizing the transformation used in \cdfourmethod{}. Let
\[
F(s(\theta_*;\x)\mid \x) = \int \mathbb{I}\{s(\theta;\x) \leq s(\theta_*;\x)\} \, p(\theta \mid \x)\, d\theta
\]
denote the true conditional cumulative distribution function (CDF) of the score \(s(\theta; \x)\), evaluated at a fixed test point \((\theta_*, \x)\). The estimated version is defined as
\[
s'(\theta_*;\x) = \widehat{F}(s(\theta_*;x)\mid \x) = \int \mathbb{I}\{s(\theta;\x) \leq s(\theta_*;\x)\} \, \widehat{p}(\theta \mid \x)\, d\theta,
\]
where \(\widehat{p}(\theta \mid \x)\) is an estimate of the conditional posterior density ${p}(\theta \mid \x)$.

In the marginal setting, the \cdfourmethod{} procedure can be interpreted as standard conformal prediction applied to a transformed score function \(s'(\theta; x) = \widehat{F}(s(\theta; x) \mid x)\). The calibration scores \(s'_i = \widehat{F}(s(\theta_i; \X_i) \mid \X_i)\) are computed from i.i.d.\ samples \((\theta_i, \X_i)\), and the conformal region is constructed using their adjusted empirical \((1 - \alpha)\)-quantile. Since the procedure follows the usual conformal recipe—changing only the score function—the marginal coverage guarantee holds by the standard conformal prediction argument.

\begin{proof}[Proof of Theorem \ref{thm:cdf_marginal}]
    The procedure applies conformal prediction to the transformed scores \(s'_i = \widehat{F}(s(\theta_i; \X_i) \mid \X_i)\), which are computed from exchangeable calibration pairs. Since the conformal region is defined by comparing the test score \(s'(\theta; \x)\) to the empirical \((1 - \alpha)\)-quantile of the calibration scores, the standard conformal guarantee ensures that
    \[
    \mathbb{P}(s'(\theta; \X) \leq t_{1 - \alpha}) \geq 1 - \alpha,
    \]
    where \((\theta, \X)\) is drawn jointly from the same distribution as the calibration data. This implies marginal coverage of the conformal region \({C}_{\text{cdf}}(\x)\).
\end{proof}

We now turn to the conditional coverage properties of \cdfourmethod{}. Unlike the marginal case, conditional validity depends on three sources of approximation: the accuracy of the posterior estimator $\widehat{p}(\theta \mid \x)$, the Monte Carlo approximation used to compute the transformed score $\widehat{F}_M(s(\theta; \x) \mid \x)$, and the estimation of the conformal quantile from a finite calibration set. To build intuition, we first consider an idealized setting in which the posterior is exact ($\widehat{p} = p$) and the conditional CDF $F(s(\theta; x) \mid x)$ is known exactly, so that the only remaining source of error comes from estimating the conformal quantile using a finite calibration set. In this setting, conditional coverage follows from the probability integral transform as $B \to \infty$, as formalized in Theorem~\ref{thm:cdf_asymptotic_ideal}.

\begin{theorem}[Asymptotic conditional coverage: idealized case]
\label{thm:cdf_asymptotic_ideal}
Suppose $\widehat{p} = p$ and the CDF transformation $F(s(\theta; \x) \mid \x)$ is known exactly. Then, as the calibration size $B \to \infty$, \cdfourmethod{} achieves asymptotic conditional coverage:
\[
\lim_{B \to \infty} \mathbb{P}\big(\theta \in \widehat{C}(\x) \mid \X = \x\big) = 1 - \alpha.
\]

\end{theorem}

\begin{proof}
If $\widehat{p} = p$, then the transformed score becomes
\[
s'(\theta; \x) = F(s(\theta; \x) \mid \x),
\]
where $F(\cdot \mid \x)$ denotes the true conditional CDF of the score. Since $\theta \sim p(\cdot \mid \x)$, it follows from the probability integral transform that
\[
s'(\theta; \x) \,\big|\, \x \;\sim\; \mathrm{Uniform}(0,1).
\]
Therefore, for any fixed $\x$,
\[
\mathbb{P}\big(s'(\theta; \x) \leq t_{1-\alpha} \mid \x\big) = t_{1-\alpha}.
\]
Moreover, if the threshold $t_{1-\alpha}$ is computed as the empirical $(1 + 1/B)(1 - \alpha)$-quantile of i.i.d.\ $\mathrm{Uniform}(0,1)$ calibration scores, then
\[
t_{1-\alpha} \to 1 - \alpha \quad \text{as } B \to \infty,
\]
which implies that the procedure achieves asymptotic conditional coverage at level $1 - \alpha$ \citep{dheur2025multi}.
\end{proof}

\bigskip

Suppose now that the posterior estimate differs from the true posterior, but that the CDF transformation $\widehat{F}(s(\theta; \x) \mid \x)$ is still evaluated exactly from its integral definition under the approximate posterior. To make the dependence on the training sample size explicit, we write $\widehat{p}_K(\theta \mid \x)$ for the posterior estimate obtained from a training set of size $K$, and keep $p(\theta \mid \x)$ for the true posterior. Accordingly, we define
\[
\widehat{F}_K(s(\theta;\x)\mid \x) \;=\; \int \mathbb{I}\{s(\theta';\x) \leq s(\theta;\x)\}\,\widehat{p}_K(\theta' \mid \x)\, d\theta'.
\]
In this setting, the transformed score $s'(\theta; \x) = \widehat{F}_K(s(\theta;\x)\mid \x)$ is no longer uniformly distributed under $p$, and conditional validity may be compromised. Our goal is to quantify the deviation from the nominal coverage level as a function of the discrepancy between $\widehat{p}_K$ and $p$.

To recover conditional coverage in the limit, we require that the divergence between $\widehat{p}_K$ and $p$ vanishes as the posterior approximation improves.

\begin{assumption}[KL convergence of the approximate posterior]
\label{ass:cdf_kl_convergence}
Let
\[
\delta_K(\x) :=  \mathrm{KL}\left( \widehat{p}_K(\cdot \mid \x) \,\|\, p(\cdot \mid \x) \right).
\]
We assume that $\widehat{p}_K(\cdot \mid \x)$ is absolutely continuous with respect to $p(\cdot \mid \x)$, and that for almost every $\x$ with respect to the distribution of $\X$, we have $\delta_K(\x) \to 0$ as $K \to \infty$.
\end{assumption}

This assumption is reasonable when $\widehat{p}_K$ is obtained through methods that aim to approximate the true posterior by minimizing a divergence such as the KL. In such cases, it is natural to expect $\delta_K(\x)$ to decrease as $K$ grows, provided the optimization is effective and the approximating family is sufficiently flexible.

\begin{theorem}[Asymptotic conditional coverage: approximate posterior case]
\label{thm:cdf_asymptotic_kl}
Under Assumption~\ref{ass:cdf_kl_convergence}, and assuming that the calibration size \(B \to \infty\), the \cdfourmethod{} procedure achieves asymptotic conditional coverage:
\[
\lim_{B \to \infty} \; \lim_{K \to \infty} \; \mathbb{P}\big(\theta \in {C}_{\text{cdf}}(\x) \,\big|\, \X = \x\big) \;=\; 1 - \alpha
\quad \text{for almost every $\x$}.
\]
\end{theorem}

\begin{proof}[Proof of Theorem \ref{thm:cdf_asymptotic_kl}]
If $\widehat{p}_K \neq p$, then the distribution of
\[
s'(\theta; \x) \;=\; \widehat{F}_K(s(\theta; \x) \mid \x)
\]
under $p(\theta \mid \x)$ is not uniform. For any fixed $\x$, the conditional coverage can be written as
\begin{align*}
    \mathbb{P}\big(s'(\theta; \x) \leq t \,\big|\, \x\big)
    &= \mathbb{P}\left(s(\theta; \x) \leq \widehat{F}^{-1}_K(t \mid \x) \,\middle|\, \x\right) \\
    &= F\left( \widehat{F}^{-1}_K(t \mid \x) \,\middle|\, \x \right),
\end{align*}
where $F(\cdot \mid \x)$ and $\widehat{F}_K(\cdot \mid \x)$ denote the CDFs of $s(\theta; \x)$ under $p$ and $\widehat{p}_K$, respectively.

Since $\widehat{p}_K(\cdot \mid \x)$ is absolutely continuous with respect to $p(\cdot \mid \x)$, Pinsker’s inequality \citep[Theorem 4.19]{boucheron2013concentration} gives
\begin{align*}
    \left| F\left( \widehat{F}^{-1}_K(t \mid \x) \mid \x \right) - t \right|
    &= \left| F\left( \widehat{F}^{-1}_K(t \mid \x) \mid \x \right)
        - \widehat{F}_K\left( \widehat{F}^{-1}_K(t \mid \x) \mid \x \right) \right| \\
    &\leq \sqrt{ \frac{1}{2} \, \mathrm{KL}\left( \widehat{p}_K(\cdot \mid \x) \,\|\, p(\cdot \mid \x) \right) } \\
    &= \sqrt{\delta_K(\x) / 2}.
\end{align*}
Therefore,
\[
\mathbb{P}\big(s'(\theta; \x) \leq t \mid \x\big)
\;\geq\; t - \sqrt{\delta_K(\x) / 2}.
\]
The correction term $\sqrt{\delta_K(\x)/2}$ quantifies the effect of posterior misspecification.  
Under Assumption~\ref{ass:cdf_kl_convergence}, we have $\delta_K(\x) \to 0$ for almost every $\x$ as $K \to \infty$, and hence
\[
\lim_{K\to \infty}\mathbb{P}\big(s'(\theta; \x) \leq t \mid \x\big)= t.
\]
The remainder of the argument is identical to the proof of Theorem~\ref{thm:cdf_asymptotic_ideal}, yielding the stated asymptotic conditional coverage.
\end{proof}

We now address the final source of approximation in the \cdfourmethod{} procedure: the Monte Carlo estimation of the transformed score. While the previous results accounted for the error introduced by approximating the posterior $\widehat{p}_K(\theta \mid \x)$, in practice the CDF $\widehat{F}_K(s(\theta; \x) \mid \x)$ must also be approximated using a finite set of posterior samples. We show that this additional source of error vanishes as the number of samples grows, completing the proof of asymptotic conditional coverage in Theorem~\ref{thm:cdf_asymptotic}.

\begin{proof}[Proof of Theorem~\ref{thm:cdf_asymptotic}]
Fix $\x$. Let $\{\theta_j\}_{j=1}^M \stackrel{\text{i.i.d.}}{\sim} \widehat{p}_K(\cdot \mid \x)$ and denote by $\widehat{F}_{K,M}(\cdot \mid \x)$ the empirical CDF based on these samples, and by $\widehat{F}_K(\cdot \mid \x)$ its population counterpart under $\widehat{p}_K$. For any $t \in (0,1)$, consider
\[
\Delta_{K,M}(t;\x) := \Big| F(\widehat{F}_{K,M}^{-1}(t \mid \x) \mid \x) - t \Big|.
\]

We decompose
\[
\Delta_{K,M}(t;\x) \leq \Big| F(\widehat{F}_{K,M}^{-1}(t)\mid \x) - F(\widehat{F}_{K}^{-1}(t)\mid \x) \Big| + \Big| F(\widehat{F}_{K}^{-1}(t)\mid \x) - t \Big|.
\]

For fixed $K$, the Glivenko–Cantelli theorem ensures $\widehat{F}_{K,M} \to \widehat{F}_K$ uniformly as $M\to\infty$, hence $\widehat{F}_{K,M}^{-1}(t) \to \widehat{F}_K^{-1}(t)$. By continuity of $F(\cdot\mid \x)$, the first difference converges to zero, so that
\[
\lim_{M\to\infty} \Delta_{K,M}(t;\x) = \Big| F(\widehat{F}_{K}^{-1}(t)\mid \x) - t \Big|.
\]

Next, under Assumption~\ref{ass:cdf_kl_convergence} the divergence $\mathrm{KL}(\widehat{p}_K(\cdot \mid \x)\,\|\,p(\cdot \mid \x))$ vanishes as $K\to\infty$, which by Pinsker’s inequality implies that the distribution functions converge. In particular,
\[
\Big| F(\widehat{F}_{K}^{-1}(t)\mid \x) - t \Big| \to 0
\quad \text{for a.e. }\x.
\]

Thus,
\[
\lim_{K\to\infty}\lim_{M\to\infty} \Delta_{K,M}(t;\x) = 0.
\]
In other words, the transformed scores converge to a uniform distribution 
in the limit, so that we are precisely in the setting of 
Theorem~\ref{thm:cdf_asymptotic_ideal}. Therefore,
\[
\mathbb{P}\big(\theta \in C_{\mathrm{cdf}}(\x) \,\big|\, X=\x\big) \to 1-\alpha
\quad \text{for a.e. }\x.
\]
\end{proof}

\newpage
\section{Experiment details}\label{appendix:exps}
\subsection{Benchmark details}
\label{appendix:bench_details}
Table \ref{tab:benchmark_tasks} details the specific characteristics of each benchmark task, including the number of dimensions, the prior used, and the availability of true posterior samples. For tasks such as \emph{Bernoulli GLM} (and Raw), \emph{SLCP} (and Distractors), \emph{Lotka-Volterra}, and \emph{SIR}, only $10$ pre-set observations with true posterior samples are available in the \texttt{sbibm} package. Meanwhile, the remaining tasks allow for a higher budget of up to $500$ observations with true posterior samples.
\begin{table}[ht]
\centering
\caption{Summary of benchmark tasks and their key characteristics.}
\label{tab:benchmark_tasks}
\renewcommand{\arraystretch}{1.3} 
\begin{adjustbox}{max width=\textwidth}
\begin{tabularx}{\linewidth}{ccccX} 
\toprule
\textbf{Task} & \textbf{Parameters} & \textbf{Prior} & $B_{\text{sim}}$ & \textbf{Description} \\
\midrule
Bernoulli GLM & 10 & Conjugate & 10 & Generalized Linear Model with Bernoulli observations and uses sufficient statistics (10-D) \\
Bernoulli GLM Raw & 10 & Conjugate & 10 & Raw data version (100-D) of Bernoulli GLM. \\
Gaussian Linear & 10 & Gaussian & 500 & Mean inference with fixed covariance and conjugate prior \\
Gaussian Linear Unif. & 10 & Uniform & 500 & Same as Gaussian Linear, but using a uniform prior. \\
Gaussian Mixture & 2 & Uniform & 500 & Bimodal mixture of Gaussians and ABC benchmark case \\
Lotka-Volterra & 4 & Lognormal & 10 & An ecological model describing the dynamics of two interacting species, such as a prey-predator relationship. \\
SIR & 2 & Lognormal & 10 & An influential three-state epidemiological model parameterized by contact rate and mean recovery. \\
SLCP & 5 & Uniform & 10 & A challenging inference task that starts from a simple likelihood and a complex posterior. \\
SCLP Distractors & 5 & Uniform & 10 & Same task as SLCP, but with additional 100 noisy features (distractors). \\
Two Moons & 2 & Uniform & 10 & Crescent-shaped posterior and tests multimodal performance \\
\bottomrule
\end{tabularx}
\end{adjustbox}
\end{table}

For the \emph{Gaussian Mixture} task, we made specific modifications to the default hyperparameters of the simulator and prior to improve its behavior and enable more fruitful comparisons. We set the uniform prior's limits to $-3$ and $3$ and changed the multiplicative factor of each mixture mean parameter (which are the parameters of interest) from $1$ to $0.8$.

\subsection{Illustration Details}
\label{appendix:illustration_details}
For both illustrations,  we used a credibility level of $1- \alpha = 0.9$ and an overall simulation budget of $B_{\text{all}} = 20000$. Of this budget, $80\%$ was used for training the posterior approximator, with the remaining $20\%$ reserved for the calibration step.
\vspace{3mm}

\textbf{Nuisance parameter illustration.} We utilized the 10-dimensional Gaussian Linear Uniform task, as detailed in Table \ref{tab:benchmark_tasks}. The first two parameters were set as our parameters of interest, $\phi = (\theta_1, \theta_2)$. The posterior estimator $\widehat{p}(\phi \mid \x)$ was a Neural Posterior Estimator (NPE) based on conditional normalizing flows from the \texttt{sbi} package, using its standard architecture. We compare the credible regions derived from the global, self-calibrated, and \cdfourmethod{} methods for the fixed observation:
$$\x_{\text{obs}} = (0.3416, -0.4812, -0.0749,  0.3471, -0.7253,  0.1747, -0.1242, -0.3328, 0.0409, -0.5498) ,$$
which was generated from the fixed full parameter vector:
$$
\theta = (0.25, 0.1, 0, 0, 0,0, 0, 0, 0, 0) ,
$$
with the true parameter of interest being the first two entries, $\phi = (0.25, 0.1)$.

\vspace{3mm}
\noindent
\textbf{Continuous-flow generative model credible set illustration.} This illustration uses the 2-dimensional Gaussian Mixture task, as detailed in Table \ref{tab:benchmark_tasks}. To derive the posterior estimator, we used a Neural Posterior Score Estimator (NPSE) based on a conditional diffusion model \citep{geffner2023compositional}, using the standard architecture from the \texttt{sbi} \citep{BoeltsDeistler_sbi_2025} implementation.  Since the NPSE is a sample-based model, we use the KDE approximation of the HPD score detailed in Algorithm \ref{alg:KDE_score} to build the set $\mathcal{D}' = \{(s(\theta_i; \x_i), \x_i)\}_{i = 1}^B$ and to compute conformal scores $s(\theta;\x)$ for any other fixed $\x$, with Scott's rule used to determine the kernel bandwidth \citep{scott2015multivariate}. We compare the self-calibrated and \locartourmethod{} credible regions for the fixed observation $\x_{\text{obs}} = (0.2651, -0.1454)$, which was generated under the fixed parameter $\theta = (0.15, -0.1)$.

\subsection{Computational Details}
\label{appendix:comp_details}
To derive the partitions for \locartourmethod{} using regression trees, we fixed the minimum number of samples per leaf at 300 for large overall budgets ($B_{\text{all}} = 10000$ or $B_{\text{all}} = 20000$) and at 75 for a small budget ($B_{\text{all}} = 2000$) to ensure well-populated partitions. Post-pruning was also performed via cost-complexity methods (\textit{ccp\_alpha}) to balance partition complexity and predictive performance, with the remaining hyperparameters set to the \textit{scikit-learn} defaults \citep{pedregosa2011scikit}. For \cdfourmethod{}, its sole hyperparameter, the sample size for estimating the empirical CDF, $M$, was fixed at $1000$.

\subsection{Additional results}
\label{appendix:additional_res}
Conditional coverage performance for NPE-based posterior estimators is shown across benchmark tasks for a smaller overall budget ($B_{\text{all}} = 2000$) and a larger one ($B_{\text{all}} = 20000$), in comparison to the original budget of $10000$. For the smaller budget, the performance of our approaches and Global CP is similar, with \cdfourmethod{} being a top performer and the best method for this case, excelling in 9 out of 10 benchmarks. For this limited sample size, the Global CP approach may be a better option than \locartourmethod{}, as its local partitioning strategy is less effective with sparse calibration data. On the other hand, when considering a larger budget, both \locartourmethod{} and \cdfourmethod{} show a clear advantage over competing approaches, with each excelling in 7 out of 10 benchmarks. The best competing method, Global CP, only performs well in 5 tasks, particularly in Lotka-Volterra. Ultimately, these results demonstrate the robustness of our method across different posterior estimator qualities and calibration budgets.
\begin{figure}[ht]
    \centering
    \includegraphics[width=1\columnwidth]{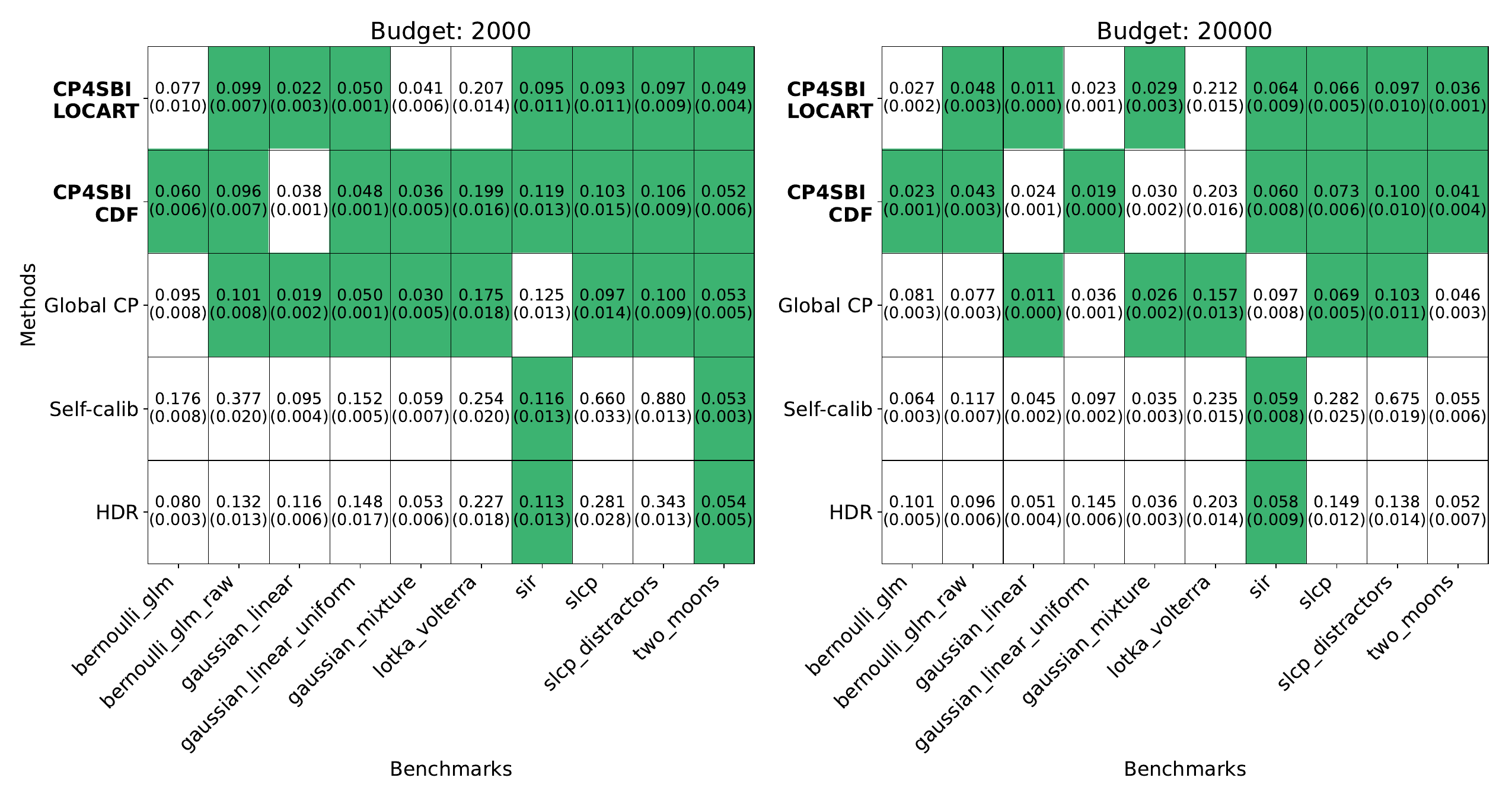}
    \caption{Conditional MAE for NPE-based posterior estimators across benchmark tasks for the overall budgets of $B_{\text{all}} = 2000$ (left) and $B_{\text{all}} = 20000$ (right).  Mean metric values and 95\% confidence intervals, calculated over 30 repetitions, are reported. Green cells indicate statistically significant superiority. For the smaller budget of $B_{\text{all}}=2000$, \cdfourmethod{} shows improved performance, standing out in 9 out of 10 benchmarks. For the larger budget of $B_{\text{all}}=20000$, both \cdfourmethod{} and \locartourmethod{} perform well, each standing out in 7 out of 10 benchmarks. This showcases how our method adapts well across different budgets.}
    \label{fig:NPE_heatmap_figure_MAE_other_budgets}
\end{figure}

As a complement to our coverage analysis, we compare credible set sizes for our 2D benchmarks (Gaussian Mixture, Two Moons, and SIR) in Table \ref{tab:hpd_areas_2d}, as volume calculation becomes prohibitive in higher dimensions. A joint analysis with the conditional MAE heatmaps (Figs \ref{fig:NPE_heatmap_figure_MAE_and_Marginal}, \ref{fig:NPE_heatmap_figure_MAE_other_budgets}) reveals the instability of non-conformal methods, whose sizes fluctuate dramatically; for instance, the Self-calibrated approach is overconfident (too small) for Gaussian Mixture but overly conservative (too large) for Two Moons, with HDR also showing erratic behavior across tasks. This contrasts with the conformal methods, where the nuanced differences highlight the adaptivity of \locartourmethod{}. Compared to the Global approach, \locartourmethod{} produces a slightly larger set for Gaussian Mixture but a noticeably smaller set for the SIR task, all of this while being a top-performer (in terms of MAE) on all these benchmarks. This showcases how \locartourmethod{} successfully adapts the region size based on the task's specific local properties, rather than defaulting to a fixed, non-adaptive behavior.
\begin{table}[h]
\caption{Computed area of HPD regions based on NPE for all 2-dimensional benchmarks.}
\label{tab:hpd_areas_2d}
\centering
\begin{adjustbox}{max width=\textwidth}
\begin{tabular}{ccccccc}
\hline
\textbf{Benchmarks}  & \textbf{\begin{tabular}[c]{@{}c@{}}Simulation \\ Budget\end{tabular}} & \textbf{\begin{tabular}[c]{@{}c@{}}CP4SBI\\ LOCART\end{tabular}} & \textbf{\begin{tabular}[c]{@{}c@{}}CP4SBI\\ CDF\end{tabular}} & \textbf{Global} & \textbf{\begin{tabular}[c]{@{}c@{}}Self\\ calibrated\end{tabular}} & \textbf{HDR}   \\ \hline
Gaussian-Mixture     & 10000                                                                 & 9.920 (0.231)                                                    & 9.891 (0.242)                                                 & 9.848 (0.241)   & 8.135 (0.312)                                                      & 12.977 (0.538) \\
Gaussian-Mixture     & 20000                                                                 & 9.566 (0.232)                                                    & 9.504 (0.246)                                                 & 9.433 (0.250)   & 7.972 (0.296)                                                      & 12.717 (0.715) \\
Two moons            & 10000                                                                 & 0.051 (0.005)                                                    & 0.060 (0.008)                                                 & 0.052 (0.006)   & 0.069 (0.006)                                                      & 0.049 (0.005)  \\
Two moons            & 20000                                                                 & 0.038 (0.003)                                                    & 0.041 (0.005)                                                 & 0.038 (0.004)   & 0.049 (0.005)                                                      & 0.035 (0.002)  \\
SIR $\times 10^{-1}$ & 10000                                                                 & 0.023 (0.004)                                                    & 0.033 (0.009)                                                 & 0.027 (0.006)   & 0.037 (0.011)                                                      & 0.055 (0.017)  \\
SIR $\times 10^{-1}$ & 20000                                                                 & 0.017 (0.005)                                                    & 0.039 (0.016)                                                 & 0.029 (0.010)   & 0.044 (0.018)                                                      & 0.066 (0.029)  \\ \hline
\end{tabular}
\end{adjustbox}
\end{table}

\end{document}